\documentclass[twoside,11pt]{article}

\usepackage{amsmath,amssymb,amsfonts}%
\usepackage{amsthm}%
\usepackage{mathrsfs}%
\usepackage{enumitem,booktabs}%
\usepackage{algorithm}%
\usepackage{algorithmicx}%
\usepackage{algpseudocode}%

\usepackage[preprint]{jmlr2e}

\usepackage{mathtools}
\newcommand{\dregret}{\textup{D-Reg}}

\newcommand{\ddgap}{\textup{D-DGap}}

\newcommand{\KL}{\textup{KL}}
\newcommand{\term}[1]{(\theequation\text{#1})}
\newcommand{\termref}[2]{\text{Equation~(\hyperref[#1]{\ref{#1}#2})}}
\usepackage{xfrac}
\usepackage{empheq,eqparbox}
\newcommand{\eqmath}[3][l]{\eqmakebox[#2][#1]{$\displaystyle\if#1l{}\fi#3$}}

\usepackage{tikz,tikz-cd,pgfplots}
\usetikzlibrary{arrows,arrows.meta,positioning,calc}

\usepackage[capitalise,noabbrev]{cleveref}
\newtheorem{assumption}[theorem]{Assumption}
\crefname{assumption}{Assumption}{Assumptions}

\usepackage{lastpage}
\jmlrheading{ }{2025}{1-\pageref{LastPage}}{8/19; Revised}{ }{21-0000}{Qing-xin Meng, Xia Lei and Jian-wei Liu}

\usepackage{orcidlink}

\ShortHeadings{A Modular Algorithm for Non-Stationary OCCO}{Meng, Lei and Liu}
\firstpageno{1}

\begin{document}

\title{A Modular Algorithm for Non-Stationary Online Convex-Concave Optimization}

\author{\name Qing-xin Meng\,\orcidlink{0000-0003-4014-7405} \email qingxin6174@gmail.com \\
       \addr College of Artificial Intelligence\\
       China University of Petroleum, Beijing\\
       Beijing, 102249, China
       \AND
       \name Xia Lei\,\orcidlink{0009-0007-6287-9793} \email leixia@cuc.edu.cn \\
       \addr State Key Laboratory of Media Convergence and Communication\\
       Communication University of China\\
       Beijing, 100024, China
       \AND
       \name Jian-wei Liu \email liujw@cup.edu.cn \\
       \addr College of Artificial Intelligence\\
       China University of Petroleum, Beijing\\
       Beijing, 102249, China}


\maketitle

\begin{abstract}
This paper investigates the problem of Online Convex-Concave Optimization, which extends Online Convex Optimization to two-player time-varying convex-concave games. The goal is to minimize the dynamic duality gap~(D-DGap), a critical performance measure that evaluates players' strategies against arbitrary comparator sequences. Existing algorithms fail to deliver optimal performance, particularly in stationary or predictable environments. To address this, we propose a novel modular algorithm with three core components: an Adaptive Module that dynamically adjusts to varying levels of non-stationarity, a Multi-Predictor Aggregator that identifies the best predictor among multiple candidates, and an Integration Module that effectively combines their strengths. Our algorithm achieves a minimax optimal D-DGap upper bound, up to a logarithmic factor, while also ensuring prediction error-driven D-DGap bounds. The modular design allows for the seamless replacement of components that regulate adaptability to dynamic environments, as well as the incorporation of components that integrate ``side knowledge'' from multiple predictors. Empirical results further demonstrate the effectiveness and adaptability of the proposed method. 
\end{abstract}


\section{Introduction}

Online Convex Optimization~\citep[OCO,][]{zinkevich2003online} provides a powerful framework for addressing dynamic challenges across a variety of real-world applications, including online learning~\citep{shwartz2012online}, resource allocation~\citep{chen2017online}, computational finance~\citep{GUO2021Adaptive}, and online ranking~\citep{Chaudhuri2017online}. It models repeated interactions between a player and the environment, where at each round $t$, the player selects $x_t$ from a convex set $X$, and the environment subsequently reveals a convex loss function $\ell_t$. The goal is to minimize dynamic regret, defined as the difference between the cumulative loss incurred by the player and that of an arbitrary comparator sequence:
\begin{equation*}
\dregret\,(u_{1:T})\coloneq\sum_{t=1}^T \ell_t\left(x_t\right)-\sum_{t=1}^T \ell_t\left(u_t\right), \qquad\forall u_t\in X. 
\end{equation*}
The minimax optimal D-Reg bound for OCO is known to be $O(\sqrt{(1 + P_T)T})$, where $P_T$ represents the path length of the comparator sequence \citep{zhang2018adaptive}. Achieving this bound typically relies on the meta-expert framework, which consists of a two-layer structure: the inner layer incorporates multiple experts, each operating a base algorithm with distinct learning rates, while the outer layer aggregates the experts' advice via weighted decision-making. \citet{zhang2018adaptive} introduced the ADER algorithm within this framework, which achieves the minimax optimal bound up to logarithmic factors.
Moreover, certain ADER-like algorithms, incorporating implicit updates \citep{Campolongo2021closer} or optimistic strategies \citep{Scroccaro2023Adaptive}, can further reduce the D-Reg bound to $O(1)$ in stationary environments or in non-stationary environments with perfect predictability. 

Online Convex-Concave Optimization~(OCCO) extends OCO by introducing two players interacting in a sequence of time-varying convex-concave games. At round $t$, the two players jointly select a strategy pair $(x_t, y_t)$ from a convex feasible set $X \times Y$, with the $x$-player minimizing and the $y$-player maximizing their respective payoffs, followed by the environment revealing a continuous convex-concave payoff function $f_t$. Both players act without prior knowledge of the current or future payoff functions. Targeting a broad spectrum of non-stationary levels, we introduce the \emph{dynamic duality gap}~(D-DGap) as the performance metric, comparing the players' strategies with an arbitrary comparator sequence in hindsight: 
\begin{equation}
\label{def:D-DGap}
\begin{aligned}
\ddgap\,(u_{1:T},v_{1:T})\coloneq\sum_{t=1}^T \Bigl(f_t\left(x_t, v_t\right)-f_t\left(u_t, y_t\right)\Bigr),\qquad \forall (u_t,v_t)\in X\times Y.
\end{aligned}
\end{equation}
Here, the D-DGap not only generalizes D-Reg from the OCO setting but also extends the classical duality gap from static convex-concave games by benchmarking performance against arbitrary comparator sequences $\{(u_t,v_t)\}_{t=1}^T$ instead of fixed worst-case comparators at each round. This flexibility allows D-DGap to capture various levels of non-stationarity metrics, from static individual regret to classical duality gap. 

The primary challenge of OCCO lies in maintaining a low D-DGap while adapting to dynamic environmental changes. To address this, we propose a modular algorithm composed of three key components: the Adaptive Module, the Multi-Predictor Aggregator, and the Integration Module. Each module plays a distinct role: 
\begin{itemize}
\item \emph{Adaptive Module:} Designed to handle varying levels of non-stationarity, this module ensures a minimax optimal D-DGap upper bound of $\widetilde{O}(\sqrt{(1+P_T)T})$. It accomplishes this by running a pair of ADER or ADER-like algorithms, which approximate the minimax optimal D-Reg. 
\item \emph{Multi-Predictor Aggregator:} This module improves decision-making by dynamically selecting the most accurate predictor. In stationary environments or non-stationary settings with perfect predictions, it guarantees a sharp $\widetilde{O}(1)$ D-DGap upper bound. This is achieved via the clipped Hedge algorithm.  
\item \emph{Integration Module:} This module unifies the Adaptive Module and the Multi-Predictor Aggregator, allowing the final strategy to adapt to a broad range of non-stationary levels while effectively tracking the best predictor. A distinctive feature of this module is its interdependent update mechanism, where the prediction-error expert and the meta-algorithm are coupled. 
\end{itemize}
Our modular design enables the interchangeable use of components that adjust to dynamic environments and the integration of modules that incorporate ``side knowledge'' from multiple predictors. 
\cref{fig:Structural} illustrates this architecture. 
Given $d$ available predictors, our algorithm guarantees:  
\begin{equation*}
\begin{aligned}
\ddgap\,(u_{1:T},v_{1:T})\leq\widetilde{O}\left(\min\left\{ V^1_T,\ \cdots,\ V^d_T,\ \sqrt{(1+P_T)T},\ \sqrt{(1+C_T)T}\right\}\right),
\end{aligned}
\end{equation*}
where $V^k_T$ quantifies the prediction error of the $k$-th predictor, $C_T$ provides an upper bound for $P_T$. 
This result not only approximates the minimax optimal D-DGap upper bound but also achieves bounds based on prediction error, with any potential improvements constrained to at most a logarithmic factor.

\begin{figure}[tb]
\centering
\resizebox{0.93\linewidth}{!}{
\footnotesize
\tikzstyle{block} = [draw, rectangle, fill=blue!4, rounded corners, minimum height=2em, minimum width=10em, align=center]
\tikzstyle{bigblock} = [draw, rectangle, fill=none, rounded corners, minimum height=2.5em, minimum width=10em, align=center]
\tikzstyle{sum} = [draw, circle, fill=none, minimum height=1em, minimum width=1em, align=center,inner sep=1.6pt]
\tikzstyle{output} = [draw=none, fill=none,minimum height=1em, minimum width=3em, align=center]
\tikzstyle{input} = [draw=none, fill=none,minimum height=1em, minimum width=1em, align=right]
\tikzstyle{point} = [coordinate]
\hspace*{-2.3em}
\begin{tikzpicture}[auto,>=latex',line/.style={-Stealth,thick}]
\node [block] (algorithm) {~Prediction-Error Expert~};
\node [sum, right=9em of algorithm] (sum) {$\sum$};
\draw [thick,line] (algorithm) -- node[midway, above] {$(\widehat{x}_t,\widehat{y}_t)$} (sum);
\node [block, below=2em of sum] (meta) {Meta-Algorithm};
\draw [thick,line] (meta) -- node[midway, left] {$(\boldsymbol{w}_t,\boldsymbol{\omega}_t)$} (sum);
\draw [densely dotted,thick] (algorithm) -- node[midway, below left =0em and -1em] {$\begin{aligned}\textup{interdependent}\\[-2pt]\textup{update}\end{aligned}$} (meta);
\node [output, right=14em of sum] (output) {Output};
\draw [line] (sum) -- node[midway, above] {$x_t = [\widehat{x}_{t}, \overline{x}_t]\boldsymbol{w}_t$}  node[midway, below] {$y_t = [\widehat{y}_{t}, \overline{y}_t]\boldsymbol{\omega}_t$} (output);
\node [bigblock, above =3em of algorithm] (adaptive) {Adaptive Module: \\~~~~run a pair of ADER or~~~~\\ADER-like algorithms};
\node [point, right=2.5em of adaptive] (point) {};
\draw [thick] (adaptive) -- (point);
\draw [thick,line] (point) -- node[midway, above left=1.8em and -1.2em] {$(\overline{x}_t,\overline{y}_t)$} (sum);
\node [bigblock, below =8em of algorithm] (aggregator) {~Multi-Predictor Aggregator~};
\draw [-Stealth] (aggregator) -- node[midway, below left=2em and 0em ] {weighted average predictor $h_t$} (algorithm);
\node [input, right=2.5em of aggregator] (input) {Predictors $\{h_t^1, h_t^2,\cdots,h_t^d\}$};
\draw [line] (input) -- (aggregator);
\draw [draw=black,fill=none, rounded corners] (-2.15,-1.9) rectangle ++(11.3,2.6);
\node [input, below left =-0.3em and 1em of algorithm] (label) {$\begin{aligned}\textup{Integration}\\[-2pt]\textup{Module:}\end{aligned}$};
\end{tikzpicture}
}
\caption{Structural Diagram of Our Modular Algorithm.}
\label{fig:Structural}
\end{figure}
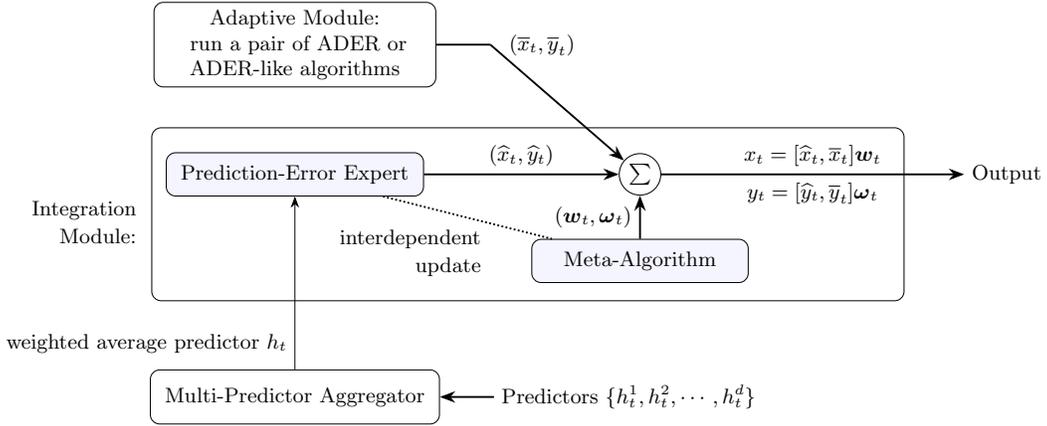

\cref{app:Experiments} provides experimental validation of our algorithm's effectiveness.

\textbf{Technical Challenges} 
The minimax optimal D-DGap bound is $O\bigl(\sqrt{(1+P_T)T}\bigr)$, where $P_T = \sum_{t=1}^T \left( \lVert u_t - u_{t-1} \rVert + \lVert v_t - v_{t-1} \rVert \right)$ is the path length of the comparator sequence (refer to Proposition~\ref{thm:lower-bound}). Approximating this bound requires each player to apply an ADER or ADER-like algorithm, which naturally ensures minimax optimality. However, in favorable scenarios such as stationary or predictable environments, we aim to further tighten the D-DGap beyond this minimax bound. Applying implicit or optimistic methods to achieve this goal introduces structural challenges. For instance, implementing ADER-like algorithms with optimistic implicit online mirror descent as the base algorithm requires using predictors $h_t(\,\cdot\,, y_t)$ for the $x$-player and $-h_t(x_t, \,\cdot\,)$ for the $y$-player. This creates a contradiction, as $(x_t, y_t)$ is computed based on the predictor $h_t$.  

To address this, we encapsulate the ADER pair into an Adaptive Module, treating it as one expert within the Integration Module, specifically designed to ensure adaptability to arbitrary comparator sequences. Additionally, we design another expert dedicated to generating a prediction error-based D-DGap. A meta-algorithm is then used to combine the strengths of both experts. Unlike traditional meta-expert frameworks, the integration module introduces an interdependent update mechanism to ensure coordinated updates between the prediction-error expert and the meta-algorithm.

\textbf{Related Work} 
D-Reg was first introduced by \citet{zinkevich2003online}, who demonstrated that greedy projection achieves a D-Reg upper bound of $O\big((1 + P_T)\sqrt{T}\big)$. To approximate the minimax optimal D-Reg of $O\big(\sqrt{(1 + P_T)T}\big)$, \citet{zhang2018adaptive} developed the ADER algorithm, which utilizes the meta-expert framework --- a two-layer structure employing multiple learning rates, as illustrated in MetaGrad~\citep{erven2016metagrad}. Since the introduction of ADER, the meta-expert framework has effectively addressed various levels of non-stationarity~\citep{lu2019adaptive, zhao2020dynamic, zhang2020online, zhang2021revisiting, zhao2021adaptivity, zhang2022simple, zhao2022efficient, lu2023nonstationary}. 
To further reduce D-Reg, \citet{Campolongo2021closer} implemented implicit updates, resulting in a D-Reg upper bound driven by the temporal variability of loss functions. Subsequently, \citet{Scroccaro2023Adaptive} refined this approach by establishing a predictor error-based D-Reg bound using optimistic implicit updates.

OCCO represents a time-varying extension of the minimax problem, which was first introduced by \cite{vNeumann1928Zur}. The seminal work of \cite{FREUND199979} connected the minimax problem to online learning, sparking interest in no-regret algorithms for static environments~\cite{Anagnostides2022near,DASKALAKIS2015327,daskalakis2021nearoptimal,nguyen2019exploiting,Syrgkanis2015fast}. Recent research has broadened this focus to time-varying games~\cite{anagnostides2023on,fiez2021online,Roy2019online}, with \cite{cardoso2018online} being the first to explicitly investigate OCCO and introduce the concept of saddle-point regret, later redefined as Nash equilibrium regret~\cite{cardoso2019competing}. \cite{zhang2022noregret} further refined the concept of dynamic Nash equilibrium regret and proposed a parameter-free algorithm that guarantees upper bounds for three metrics: static individual regret, duality gap, and dynamic Nash equilibrium regret. 
More recently, \cite{meng2025proximal} highlighted potential limitations in relying on dynamic Nash equilibrium regret as a performance metric.

This paper extends implicit updates, optimistic techniques, and meta-expert frameworks --- primarily applied in OCO --- to OCCO. Our algorithm introduces a modular design with an interdependent update mechanism. In contrast to \citet{zhang2022noregret}, which optimizes separate metrics without ensuring their tightness, our approach unifies these measures under D-DGap and provides a rigorous tightness guarantee.

\section{Preliminaries}

Let $\mathscr{X}$ and $\mathscr{Y}$ be finite-dimensional Euclidean spaces. 
The Fenchel coupling~\citep{Mertikopoulos2016Learning,2016arXiv160807310M} induced by a proper function $\varphi$ is defined as 
$B_{\varphi}\left(x, z\right)\coloneq \varphi\left(x\right) + \varphi^\star\left(z\right) - \left\langle z, x \right\rangle$, $\forall (x, z) \in \mathscr{X} \times \mathscr{X}^*$, 
where $\varphi^\star$ represents the convex conjugate of $\varphi$, given by
$\varphi^\star\left(z\right)\coloneq \sup_{x\in \mathscr{X}} \left\{\langle z, x \rangle - \varphi\left(x\right) \right\}$, 
and the bilinear map $\langle\,\cdot\,,\cdot\,\rangle\colon \mathscr{X}^*\times \mathscr{X}\rightarrow\mathbb{R}$ denotes the canonical dual pairing. Here, $\mathscr{X}^*$ is the dual space of $\mathscr{X}$. 
Fenchel coupling extends the concept of Bregman divergence to more complex primal-dual settings.  According to the Fenchel-Young inequality, we have $B_{\varphi}\left(x, z\right)\geq 0$, with equality holding if and only if $z$ is a subgradient of $\varphi$ at $x$. To simplify notation, we use $x^\varphi$ to denote one such subgradient of $\varphi$ at $x$. By directly applying the definition of Fenchel coupling, we obtain 
$B_{\varphi}\left(x, y^\varphi\right) + B_{\varphi}\left(y, z\right) - B_{\varphi}\left(x, z\right) = \left\langle z - y^\varphi, x - y \right\rangle$. 
%
A function $\varphi$ is called $\mu$-strongly convex if 
$B_{\varphi}\left(x, y^\varphi\right)\geq\frac{\mu}{2}\left\lVert x-y\right\rVert^2$, $\forall x, y\in\mathscr{X}$.

The standard simplex refers to the set of all non-negative vectors that sum to 1, defined as 
$\bigtriangleup_{d} \coloneq \{\boldsymbol{w} \in \mathbb{R}_+^d \mid \lVert \boldsymbol{w} \rVert_1 = 1\}$. 
The clipped version modifies this by restricting the elements of $\boldsymbol{w}$ to lie within a predefined range, resulting in 
$\bigtriangleup_d^\alpha \coloneq \{\boldsymbol{w} \in \mathbb{R}_+^d \mid \lVert \boldsymbol{w} \rVert_1 = 1,\ w^i \geq \alpha/d,\ \forall i = 1,2,\cdots,d\}$, 
where $\alpha$ represents the clipping coefficient. 
The Kullback-Leibler (KL) divergence can be viewed as a specific case of Fenchel coupling, induced by the negative entropy, a 1-strongly convex function. As a result, we have the inequality 
$\KL(\boldsymbol{w}, \boldsymbol{u}) \geq \frac{1}{2}\left\lVert \boldsymbol{w}-\boldsymbol{u} \right\rVert_1^2$, $\forall \boldsymbol{w}, \boldsymbol{u}\in\bigtriangleup_{d}$. 

We use big $O$ notation for asymptotic upper bounds and $\widetilde{O}$ to omit polylogarithmic terms.

\section{Main Results}

In this section, we first formalize the OCCO framework and outline assumptions. Subsequently, we analyze the Adaptive Module, the Integration Module, and the Multi-Predictor Aggregator in detail. Finally, we elucidate the logical structure of our algorithm and highlight its performance advantages. 

\subsection{Problem Formalization}

OCCO can be formalized as follows: At round $t$, 
\begin{itemize}[itemsep=0pt,topsep=4pt]
\item \emph{Actions:} $x$-player chooses $x_t\in X$ and $y$-player chooses $y_t\in Y$, where the feasible sets $0\in X\subset\mathscr{X}$ and $0\in Y\subset\mathscr{Y}$ are both compact and convex. 
\item \emph{Feedback:} The environment feeds back $f_t\colon X\times Y\rightarrow \mathbb{R}$, where $f_t$ is continuous and $f_t\left(\,\cdot\,,y\right)$ is convex in $X$ for every $y\in Y$ and $f_t\left(x, \,\cdot\,\right)$ is concave in $Y$ for every $x\in X$. 
\end{itemize}
The goal is to minimize D-DGap. 
Similar to previous studies in online learning, we introduce the following standard assumptions. 
\begin{assumption}
\label{ass:X-Y-bounded}
The diameter of $X$ is denoted as $D_X$, and the diameter of $\,Y$ is denoted as $D_Y$, that is, $\forall x, x'\in X$, $\forall y, y'\in Y$, the following inequalities hold: 
\begin{equation*}
\begin{aligned}
\lVert x-x'\rVert\leq D_X,\qquad
\lVert y-y'\rVert\leq D_Y.
\end{aligned}
\end{equation*}
\end{assumption}
\begin{assumption}
\label{ass:2-subgradient-bounded}
All payoff functions are bounded, and their subgradients are also bounded. Specifically, $\exists\,M$, $G_X$ and $G_Y$, such that $\forall x\in X$, $\forall y\in Y$ and $\forall t$, the following inequalities hold: 
\begin{equation*}
\begin{aligned}
\lvert f_t\left(x,y\right)\rvert\leq M,\qquad
\left\lVert \nabla_x f_t\left(x,y\right)\right\rVert\leq G_X,\qquad
\lVert \nabla_y (-f_t)\left(x,y\right)\rVert\leq G_Y.
\end{aligned}
\end{equation*}
\end{assumption}

\subsection{Adaptive Module}

Before introducing the Adaptive Module, we establish a lower bound for D-DGap, supported by the following proposition.

\begin{proposition}[D-DGap Lower Bound]
\label{thm:lower-bound}  
For any strategies adopted by the players, there exists a sequence of convex-concave payoff functions satisfying Assumption~\ref{ass:2-subgradient-bounded}, along with a comparator sequence whose path length is given by $P_T = \sum_{t=1}^T\big(\lVert u_t - u_{t-1}\rVert + \lVert v_t - v_{t-1}\rVert\big)$, such that $P_T \leq P$. Under these conditions, the resulting D-DGap is guaranteed to be at least $\mathit{\Omega}\big(\sqrt{(1+P)T}\big)$.
\end{proposition}
\begin{proof}\emph{(Proof Sketch of Proposition~\ref{thm:lower-bound})}
The D-DGap is composed of the sum of two individual D-Regs. According to Theorem~2 of \citet{zhang2018adaptive}, in adversarial environments, no online algorithm can bound the individual D-Regs below $\mathit{\Omega}\big(\sqrt{(1 + P^u)T}\big)$ and $\mathit{\Omega}\big(\sqrt{(1 + P^v)T}\big)$, respectively, where $P^u \geq \sum_{t=1}^T\lVert u_t - u_{t-1}\rVert$ and $P^v \geq \sum_{t=1}^T\lVert v_t - v_{t-1}\rVert$. This implies that the D-DGap lower bound cannot be less than the sum of these two bounds. 
\end{proof}

To adapt to varying levels of non-stationarity and approach the D-DGap lower bound, one can utilize a pair of ADER algorithms~\citep{zhang2018adaptive} or ADER-like algorithms that replace the base algorithm with implicit updates~\citep{Campolongo2021closer} or incorporate optimistic strategies~\citep{Scroccaro2023Adaptive}. These methods align with the meta-expert framework. The following proposition demonstrates the performance guarantee when decomposing the OCCO problem into two OCO problems and running two independent ADER or ADER-like algorithms. 

\begin{proposition}[Performance for the Adaptive Module]
\label{prop:Adaptive}
Consider running two independent ADER or ADER-like algorithms, each designed to approximate the minimax optimal D-Reg. In round $t$, one algorithm outputs $\overline{x}_t$ and receives a convex loss function $f_t(\,\cdot\,, y_t)$, while the other produces $\overline{y}_t$ and receives $-f_t(x_t,\,\cdot\,)$. Under Assumptions~\ref{ass:X-Y-bounded} and~\ref{ass:2-subgradient-bounded}, we have that $\forall(u_t, v_t) \in X \times Y$:
\begin{equation*}
\begin{aligned}
\sum_{t=1}^T \Bigl(f_t(x_t, v_t) - f_t(x_t, \overline{y}_t)\Bigr) 
+ \sum_{t=1}^T \Bigl(f_t(\overline{x}_t, y_t) - f_t(u_t, y_t)\Bigr) 
\leq \widetilde{O}\big(\sqrt{(1 + \min\{P_T, C_T\})T}\big),
\end{aligned}
\end{equation*}
where $C_T = \sum_{t=1}^T\big(\left\lVert x'_t - x'_{t-1}\right\rVert + \left\lVert y'_t - y'_{t-1}\right\rVert\big)$ serves as the effective upper threshold of $P_T$, $x'_t = \arg\min_{x \in X} f_t(x, y_t)$, and $y'_t = \arg\max_{y \in Y} f_t(x_t, y)$.
\end{proposition}

$P_T$ represents the path length of the comparator sequence, reflecting the assumed level of environmental non-stationarity. It ranges from $0$ to linear growth, allowing for various scenarios. In contrast, $C_T$ is a data-dependent measure that reflects the worst-case non-stationarity observed during the interactions between the players and the environment. It serves as an effective upper threshold for $P_T$, as shown by the inequality $f_t(x_t, v_t) - f_t(u_t, y_t) \leq f_t(x_t, y'_t) - f_t(x'_t, y_t)$. 
After $T$ rounds of the game, only those comparator sequences with path lengths satisfying $P_T \leq C_T$ are considered meaningful.  
Unlike single-player setups, where $C_T = \sum_{t=1}^T \lVert x^*_t - x^*_{t-1} \rVert$ (with $x_t^* = \arg\min_x \ell_t(x)$), which depends solely on the environment, in two-player settings, $C_T$ becomes algorithm-dependent. It captures the mutual influence of the strategies adopted by both players.

The two ADER or ADER-like algorithms outlined in Proposition~\ref{prop:Adaptive} operate independently, with each player's output influencing the other's loss function. This mutual dependence complicates further tightening of the D-DGap upper bound, particularly in favorable scenarios such as stationary or predictable environments. In the OCCO setting, two players can coordinate strategies to better adapt to environmental dynamics. Consequently, the outputs $\overline{x}_t$ and $\overline{y}_t$ from Proposition~\ref{prop:Adaptive} are not directly used as the final strategies. Instead, the method serves as an \emph{Adaptive Module}, capable of handling diverse levels of non-stationarity. In the next section, we explore how the two players can further collaborate to refine their strategies.

\subsection{Integration Module}

The objective of this section is to design an algorithm that not only 1) automatically adapts to arbitrary comparator sequences but also 2) guarantees a D-DGap upper bound based on prediction error.
To begin with, we first consider a simplified problem: how to achieve these two objectives separately. For the first objective, simply running a pair of ADER or ADER-like algorithms is sufficient. For the second objective, we need to explore the following updates:
\begin{equation}
\label{eq:optoppm}
\begin{aligned}
(x_{t},y_{t})&=\arg\min\nolimits_{x\in X}\max\nolimits_{y\in Y} \eta_t \gamma_t h_t(x,y)+\gamma_t B_{\phi}\big(x, \widetilde{x}_t^\phi\big)-\eta_t B_{\psi}\big(y, \widetilde{y}_t^\psi\big), \\
\widetilde{x}_{t+1}&=\arg\min\nolimits_{x\in X} \eta_t f_t(x,y_t)+B_{\phi}\big(x, \widetilde{x}_{t}^\phi\big), \\
\widetilde{y}_{t+1}&=\arg\max\nolimits_{y\in Y} \gamma_t f_t(x_t,y)-B_{\psi}\big(y, \widetilde{y}_{t}^\psi\big),
\end{aligned}
\end{equation}
where $h_t$ is an arbitrary convex-concave predictor, $\eta_t>0$ and $\gamma_t>0$ are learning rates. 
To facilitate our analysis, we assume that the regularizers $\phi$ and $\psi$ are both $1$-strongly convex and have Lipschitz-continuous gradients, and their Fenchel couplings satisfy Lipschitz continuity with respect to the first variable, i.e., $\exists L_\phi, L_\psi, L_{B_\phi}, L_{B_\psi} < +\infty$, $\forall \alpha, x, x' \in X$, $\forall \beta, y, y' \in Y$: 
\begin{equation*}
\begin{aligned}
\left\lVert \nabla\phi(x)-\nabla\phi(x')\right\rVert&\leq L_\phi\left\lVert x-x'\right\rVert,\quad&
\left\lvert B_\phi(x,\alpha^\phi)-B_\phi(x',\alpha^\phi)\right\rvert&\leq L_{B_\phi}\left\lVert x-x'\right\rVert, \\
\left\lVert \nabla\psi(y)-\nabla\psi(y')\right\rVert&\leq L_\psi\left\lVert y-y'\right\rVert,\quad&
\left\lvert B_\psi(y,\beta^\psi)-B_\psi(y',\beta^\psi)\right\rvert&\leq L_{B_\psi}\left\lVert y-y'\right\rVert.
\end{aligned}
\end{equation*}
These assumptions are consistent with previous literature~\citep{Campolongo2021closer,zhang2022noregret}. 

\cref{eq:optoppm} can be seen as an optimistic variant of the proximal point method or as the two-player optimistic counterpart of \citet{Campolongo2021closer}. The following lemma establishes its performance guarantee.
\begin{lemma}
\label{lem:OptOPPM}
Under Assumptions~\ref{ass:X-Y-bounded} and~\ref{ass:2-subgradient-bounded}, and let the predictor $h_t$ satisfy Assumption~\ref{ass:2-subgradient-bounded}. 
Suppose there exists $\lambda$ and $\mu$, such that $\lambda\geq\sum_{t=1}^T\left\lVert u_t - u_{t-1}\right\rVert$ and $\mu\geq\sum_{t=1}^T\left\lVert v_t - v_{t-1}\right\rVert$, then we may set learning rates as follows:
\begin{equation*}
\begin{aligned}
\eta_t&=\textstyle L_{B_\phi}(D_X+\lambda)\big/\bigl(\epsilon+\sum_{\tau=1}^{t-1}\nu_{\tau}^x\bigr), 
\qquad \gamma_t=L_{B_\psi}(D_Y+\mu)\big/\bigl(\epsilon+\sum_{\tau=1}^{t-1}\nu_{\tau}^y\bigr), \\
0\leq\nu_t^x&=f_t(x_t,y_t)-h_t(x_t,y_t)+h_t(\widetilde{x}_{t+1},y_t)-f_t(\widetilde{x}_{t+1}, y_t) -B_{\phi}\big(\widetilde{x}_{t+1}, x_t^\phi\big)/\eta_t,\\[-1pt]
0\leq\nu_t^y&=f_t(x_t,\widetilde{y}_{t+1})-h_t(x_t,\widetilde{y}_{t+1})+h_t(x_t,y_t)-f_t(x_t,y_t) -B_{\psi}\big(\widetilde{y}_{t+1}, y_t^\psi\big)/\gamma_t,
\end{aligned}
\end{equation*}
where $\epsilon > 0$ prevents initial learning rates from being infinite. 
As a result, \cref{eq:optoppm} achieves 
\begin{equation*}
\begin{aligned}
\ddgap\,(u_{1:T},v_{1:T})\leq O\left(\min\left\{\sum_{t=1}^{T}\rho\left(f_{t}, h_{t}\right),\, \sqrt{(1+\lambda+\mu)T}\right\}\right),
\end{aligned}
\end{equation*}
where $\rho(f_t,h_t)=\max_{x\in X,\,y\in Y}\left\lvert f_t(x,y)-h_t(x,y)\right\rvert$ measures the distance between $f_t$ and $h_t$, modeling the prediction error in round $t$. 
\end{lemma}

The simplified problem outlined at the beginning of this section has now been effectively addressed. Specifically, deploying a pair of ADER or ADER-like algorithms enables automatic adaptation to arbitrary comparator sequences while approximating the minimax optimal D-DGap. Additionally, leveraging \cref{eq:optoppm} ensures a prediction error-based D-DGap upper bound.

We now focus on designing the Integration Module, which aims to combine the strengths of both worlds: harnessing the adaptive capabilities of ADER or ADER-like algorithms while simultaneously ensuring a prediction error-driven D-DGap upper bound. Achieving this dual objective requires effectively integrating the strengths of both approaches into a unified framework.

To address this challenge, we introduce a tailored variant of the meta-expert framework. Its components and design criteria are as follows: 
\begin{itemize}[itemsep=0pt,topsep=4pt]
\item \emph{Expert-Layer:} Comprising two experts:
\begin{itemize}[itemsep=0pt,topsep=0pt]
\item \emph{Adaptive Module:} This expert generates the strategy pair $(\overline{x}_t, \overline{y}_t)$ and is detailed in Proposition~\ref{prop:Adaptive}.
\item \emph{Prediction-Error Expert:} 
This expert produces the strategy pair $(\widehat{x}_t, \widehat{y}_t)$ to guarantee a prediction error-based D-DGap upper bound. Its update is referred to as the \emph{expert update}.
\end{itemize}
\item \emph{Meta-Layer:} The meta-layer produces weight parameters $\boldsymbol{w}_t= [w_t, 1-w_t]^\mathrm{T}$ and $\boldsymbol{\omega}_t= [\omega_t, 1-\omega_t]^\mathrm{T}$, which balance the experts' strategies, resulting in the outputs $x_t = [\widehat{x}_t, \overline{x}_t]\boldsymbol{w}_t$ and $y_t = [\widehat{y}_t, \overline{y}_t]\boldsymbol{\omega}_t$. The meta-layer update (refer to as the \emph{meta update}) must ensure compatibility with both experts while simultaneously guaranteeing a prediction error-driven static duality gap upper bound and maintaining minimax optimality.
\end{itemize}

To fulfill the above requirements, we design both the expert update and meta update by modifying \cref{eq:optoppm}. Specifically: let the convex-concave predictor $h_t$ serve as a hint for both players, define 
\begin{equation*}
\begin{aligned}
\boldsymbol{A}_t(x,y)=
\begin{bmatrix}
f_t(x,y), & f_t(x,\overline{y}_t) \\
f_t(\overline{x}_t,y), & f_t(\overline{x}_t,\overline{y}_t)
\end{bmatrix},
\qquad
\boldsymbol{\mathit{\Lambda}}_{t}(x,y)=
\begin{bmatrix}
h_t(x,y), & h_t(x,\overline{y}_t) \\
h_t(\overline{x}_t,y), & h_t(\overline{x}_t,\overline{y}_t)
\end{bmatrix},
\end{aligned}
\end{equation*}
and let $\boldsymbol{w} = [w, 1-w]^\mathrm{T}$, $\boldsymbol{\omega} = [\omega, 1-\omega]^\mathrm{T}$, $\boldsymbol{A}_t=\boldsymbol{A}_t(\widehat{x}_{t},\widehat{y}_{t})$, $\boldsymbol{\mathit{\Lambda}}_{t}=\boldsymbol{\mathit{\Lambda}}_{t}(\widehat{x}_{t},\widehat{y}_{t})$. 
For the \emph{expert update:} 
\begin{subequations}
\begin{empheq}[left={\hspace*{0em}}]{align}
  \eqmath[r]{A}{\,H_t\left(x,y;\boldsymbol{w},\boldsymbol{\omega}\right)} & \eqmath[l]{B}{= \eta_t \gamma_t\boldsymbol{w}^\mathrm{T}\boldsymbol{\mathit{\Lambda}}_{t}(x,y)\,\boldsymbol{\omega}+ w\,\gamma_t B_{\phi}\big(x, \widetilde{x}_t^{\phi}\big)- \omega\,\eta_t B_{\psi}\big(y, \widetilde{y}_t^{\psi}\big),} \label{eq:coupled-Ht}\\
  \eqmath[r]{A}{(\widehat{x}_{t},\widehat{y}_{t})} & \eqmath[l]{B}{=\arg\min\nolimits_{x\in X}\max\nolimits_{y\in Y} H_t\left(x,y;\boldsymbol{w}_t,\boldsymbol{\omega}_t\right),} \label{eq:coupled-xy}\\
  \eqmath[r]{A}{\widetilde{x}_{t+1}} & \eqmath[l]{B}{=\arg\min\nolimits_{x\in X}\eta_t \boldsymbol{A}_t^{1,\colon\!}(x,\widehat{y}_t)\,\boldsymbol{\omega}_t+B_{\phi}\big(x, \widetilde{x}_{t}^{\phi}\big),} \label{eq:x-tilde}\\
  \eqmath[r]{A}{\widetilde{y}_{t+1}} & \eqmath[l]{B}{=\arg\max\nolimits_{y\in Y}\gamma_t \boldsymbol{w}_t^\mathrm{T} \boldsymbol{A}_t^{\colon\!,1}(\widehat{x}_t,y)-B_{\psi}\big(y, \widetilde{y}_{t}^{\psi}\big),} \label{eq:y-tilde}
\end{empheq}
\end{subequations}
\begin{subequations}
and for the \emph{meta update:} 
\begin{empheq}[left={\hspace*{0em}}]{align}
  \eqmath[r]{A}{\,W_t\left(\boldsymbol{w},\boldsymbol{\omega};x,y\right)} & \eqmath[l]{B}{= \theta_t\vartheta_t\boldsymbol{w}^\mathrm{T}\boldsymbol{\mathit{\Lambda}}_{t}(x,y)\,\boldsymbol{\omega}+ \vartheta_t \KL (\boldsymbol{w},\widetilde{\boldsymbol{w}}_t)- \theta_t \KL (\boldsymbol{\omega},\widetilde{\boldsymbol{\omega}}_t),} \label{eq:coupled-Wt}\\
  \eqmath[r]{A}{(\boldsymbol{w}_t,\boldsymbol{\omega}_t)} & \eqmath[l]{B}{=\arg\min\nolimits_{\boldsymbol{w}\in \bigtriangleup_2^\alpha}\max\nolimits_{\boldsymbol{\omega}\in \bigtriangleup_2^\alpha} W_t\left(\boldsymbol{w},\boldsymbol{\omega};\widehat{x}_{t},\widehat{y}_{t}\right),} \label{eq:coupled-wo}\\
  \eqmath[r]{A}{\widetilde{\boldsymbol{w}}_{t+1}} & \eqmath[l]{B}{=\arg\min\nolimits_{\boldsymbol{w}\in \bigtriangleup_2^\alpha} \theta_t \boldsymbol{w}^\mathrm{T}\boldsymbol{A}_t\, \boldsymbol{\omega}_t+\KL (\boldsymbol{w},\widetilde{\boldsymbol{w}}_t),} \label{eq:w-tilde}\\
  \eqmath[r]{A}{\widetilde{\boldsymbol{\omega}}_{t+1}} & \eqmath[l]{B}{=\arg\max\nolimits_{\boldsymbol{\omega}\in \bigtriangleup_2^\alpha} \vartheta_t \boldsymbol{w}_t^\mathrm{T}\boldsymbol{A}_t\, \boldsymbol{\omega}-\KL (\boldsymbol{\omega},\widetilde{\boldsymbol{\omega}}_t),} \label{eq:o-tilde}
\end{empheq}
\end{subequations}
where $\eta_t>0$, $\gamma_t>0$, $\theta_t>0$ and $\vartheta_t>0$ are learning rates, and $\alpha=2/T$. 

We designate the aforementioned updates as the \emph{Integration Module}. Crucially, the expert advice (Equation~\ref{eq:coupled-xy}) and the meta-layer weights (Equation~\ref{eq:coupled-wo}) must be updated in a coordinated manner. Specifically, the expert update requires access to the meta-layer's weights to refine its recommendations, while the meta-layer update relies on the experts' advice to adjust its weights. This interdependent update mechanism represents a significant departure from conventional meta-expert frameworks.

We defer the discussion of the coordinated update methodology and focus first on how these updates implement the functionality of the Integration Module. Specifically, \cref{thm:expert-layer} establishes that the expert update ensures a prediction error-based bound. \cref{thm:meta-layer} demonstrates that the meta-layer guarantees a prediction error-driven static duality gap upper bound while maintaining minimax optimality. 
Finally, \cref{thm:Integration} and its proof confirm that the meta-layer is compatible with both experts, simultaneously achieving a prediction error-driven D-DGap upper bound and maintaining minimax optimality.


%
\begin{theorem}[Performance for the Expert Update]
\label{thm:expert-layer}
Under Assumptions~\ref{ass:X-Y-bounded} and~\ref{ass:2-subgradient-bounded}, and let the predictor $h_t$ satisfy Assumption~\ref{ass:2-subgradient-bounded}. If the learning rates satisfy the following equations: 
\begin{equation*}
\begin{aligned}
\eta_t=\ &\textstyle L_{B_\phi} D_X (T+1)\big/\big(\epsilon+\sum_{\tau=1}^{t-1}\delta_{\tau}^x\big), \qquad
\gamma_t= L_{B_\psi} D_Y (T+1)\big/\big(\epsilon+\sum_{\tau=1}^{t-1}\delta_{\tau}^y\big), \\
0\leq \delta_t^x =\ & \bigl[f_t(\widehat{x}_t, \widehat{y}_t),\ f_t(\widehat{x}_t,\overline{y}_t)\bigr]\boldsymbol{\omega}_t-\bigl[h_t(\widehat{x}_t, \widehat{y}_t),\ h_t(\widehat{x}_t,\overline{y}_t)\bigr]\boldsymbol{\omega}_t \\
&+ \bigl[h_t(\widetilde{x}_{t+1}, \widehat{y}_t),\ h_t(\widetilde{x}_{t+1},\overline{y}_t)\bigr]\boldsymbol{\omega}_t - \bigl[f_t(\widetilde{x}_{t+1}, \widehat{y}_t),\ f_t(\widetilde{x}_{t+1},\overline{y}_t)\bigr]\boldsymbol{\omega}_t, \\
0\leq \delta_t^y =\ & \bigl[h_t(\widehat{x}_t, \widehat{y}_t),\ h_t(\overline{x}_t,\widehat{y}_t)\bigr]\boldsymbol{w}_t-\bigl[f_t(\widehat{x}_t, \widehat{y}_t),\ f_t(\overline{x}_t,\widehat{y}_t)\bigr]\boldsymbol{w}_t \\
&+ \bigl[f_t(\widehat{x}_{t}, \widetilde{y}_{t+1}),\ f_t(\overline{x}_{t},\widetilde{y}_{t+1})\bigr]\boldsymbol{w}_t - \bigl[h_t(\widehat{x}_{t}, \widetilde{y}_{t+1}),\ h_t(\overline{x}_{t},\widetilde{y}_{t+1})\bigr]\boldsymbol{w}_t, 
\end{aligned}
\end{equation*}
where $\epsilon > 0$ prevents initial learning rates from being infinite. 
Then the following inequality holds: 
\begin{equation*}
\begin{aligned}
\sum_{t=1}^{T} \Bigl(f_t(x_t, v_t) - \boldsymbol{w}_t^\mathrm{T} \boldsymbol{A}_t^{\colon\!,1}\Bigr)+\sum_{t=1}^{T} \Bigl(\boldsymbol{A}_t^{1,\colon\!} \boldsymbol{\omega}_t - f_t(u_t, y_t)\Bigr)
\leq O\left(\sum_{t=1}^{T} \rho(f_t,h_t)\right).
\end{aligned}
\end{equation*}
\end{theorem}
\begin{theorem}[Performance for the Meta Update]
\label{thm:meta-layer}
Under Assumption~\ref{ass:2-subgradient-bounded}, let the predictor $h_t$ satisfy Assumption~\ref{ass:2-subgradient-bounded}, and assume that $T\geq 2$. 
If the learning rates satisfy the following inequalities: 
\begin{equation*}
\begin{aligned}
\theta_t &=\textstyle (\ln T)\big/\bigl(\epsilon + \sum_{\tau=1}^{t-1}\Delta_{\tau}^x\bigr), 
&0\leq\Delta_{t}^x &= (\boldsymbol{w}_t-\widetilde{\boldsymbol{w}}_{t+1})^\mathrm{T}\left(\boldsymbol{A}_{t}-\boldsymbol{\mathit{\Lambda}}_{t}\right)\boldsymbol{\omega}_t-\KL (\widetilde{\boldsymbol{w}}_{t+1}, \boldsymbol{w}_t)/\theta_t, \\
\vartheta_t &=\textstyle (\ln T)\big/\bigl(\epsilon + \sum_{\tau=1}^{t-1}\Delta_{\tau}^y\bigr), 
&0\leq\Delta_{t}^y &= -\boldsymbol{w}_t^\mathrm{T}\left(\boldsymbol{A}_{t}-\boldsymbol{\mathit{\Lambda}}_{t}\right)(\boldsymbol{\omega}_t-\widetilde{\boldsymbol{\omega}}_{t+1})-\KL (\widetilde{\boldsymbol{\omega}}_{t+1}, \boldsymbol{\omega}_t)/\vartheta_t, 
\end{aligned}
\end{equation*}
where $\epsilon > 0$ prevents initial learning rates from being infinite. 
Then the meta layer of the Integration Module enjoys the following inequality: 
\begin{equation*}
\begin{aligned}
\sum_{t=1}^T \Bigl(\boldsymbol{w}_t^\mathrm{T} \boldsymbol{A}_t \boldsymbol{v} -  \boldsymbol{u}^\mathrm{T} \boldsymbol{A}_t \boldsymbol{\omega}_t \Bigr)
\leq O\left(\min\left\{ \sum_{t=1}^{T} \rho(f_t,h_t),\, \sqrt{(1 + \ln T) T} \right\}\right),\qquad\forall \boldsymbol{u}, \boldsymbol{v}\in\bigtriangleup_2.
\end{aligned}
\end{equation*}
\end{theorem}
We stress that the proofs of \cref{thm:expert-layer,thm:meta-layer} does not follow directly from an application of Lemma~\ref{lem:OptOPPM}, although they are similar in form. Proofs are reported in Appendix~\ref{app:SuppProofs}. 
Regarding learning rates configurations: Since the expert update primarily focuses on the prediction error-type upper bound, we establish a preset upper bound of the path length of comparator sequences proportional to time horizon $T$ to determine the learning rates $\eta_t$ and $\gamma_t$. As the meta update specifically addresses the static duality gap, we set time-invariant comparator sequences (with its path length being $0$) when deriving the learning rates $\theta_t$ and $\vartheta_t$. 
\begin{theorem}[D-DGap for the Integration Module]
\label{thm:Integration}
Under the settings of Proposition~\ref{prop:Adaptive} and \cref{thm:expert-layer,thm:meta-layer}, the Integration Module achieves the following D-DGap: 
\begin{equation*}
\begin{aligned}
\ddgap\,(u_{1:T},v_{1:T})\leq \widetilde{O}\left(\min\left\{\sum_{t=1}^T \rho(f_t, h_t),\ \sqrt{\bigl(1 + \min\{P_T, C_T\}\bigr)T}\right\}\right).
\end{aligned}
\end{equation*}
\end{theorem}
\begin{proof}
By employing the prediction-error expert, the D-DGap can be equivalently written as:
\begin{equation*}
\begin{aligned}
&\ddgap\,(u_{1:T},v_{1:T}) \\
&\ = \sum_{t=1}^{T} \left(\Bigl(f_t(x_t, v_t) - \boldsymbol{w}_t^\mathrm{T} \boldsymbol{A}_t^{\colon\!,1}\Bigr)
+\left(\boldsymbol{w}_t^\mathrm{T} \boldsymbol{A}_t \!\begin{bmatrix}1\\0\end{bmatrix} -  [1, 0] \boldsymbol{A}_t \boldsymbol{\omega}_t \right) 
+\Bigl(\boldsymbol{A}_t^{1,\colon\!} \boldsymbol{\omega}_t - f_t(u_t, y_t)\Bigr)\right).
\end{aligned}
\end{equation*}
Invoking \cref{thm:expert-layer,thm:meta-layer} then yields 
\begin{equation}
\label{eqpf:bound1}
\begin{aligned}
\ddgap_T \leq O\left(\sum_{t=1}^T \rho(f_t, h_t)\right).
\end{aligned}
\end{equation}
Moreover, by leveraging the adaptive module we obtain the following upper bound:
\begin{equation*}
\begin{aligned}
&\ddgap\,(u_{1:T},v_{1:T}) \\
&\ \leq \sum_{t=1}^T \left(\Bigl(f_t\left(x_t, v_t\right)-f_t\left(x_t, \overline{y}_t\right)\Bigr) 
+\left(\boldsymbol{w}_t^\mathrm{T} \boldsymbol{A}_t \!\begin{bmatrix}0\\1\end{bmatrix} -  [0, 1] \boldsymbol{A}_t \boldsymbol{\omega}_t \right) 
+\Bigl(f_t\left(\overline{x}_t, y_t\right)-f_t\left(u_t, y_t\right)\Bigr)\right).
\end{aligned}
\end{equation*}
Applying Proposition~\ref{prop:Adaptive} together with \cref{thm:meta-layer} gives 
\begin{equation}
\label{eqpf:bound2}
\begin{aligned}
\ddgap_T \leq \widetilde{O}\left(\sqrt{(1 + \min\{P_T, C_T\})T}\right).
\end{aligned}
\end{equation}
Combining \cref{eqpf:bound1,eqpf:bound2} yields the claimed result. 
\end{proof}

Having analyzed the Integration Module's functionality, we now introduce our joint solution method for \cref{eq:coupled-xy,eq:coupled-wo}. The following theorem guarantees that this coupled system admits a unique solution. 
\begin{theorem}
\label{thm:existence}
There exists a unique solution to the coupled system given by \cref{eq:coupled-xy,eq:coupled-wo}. 
\end{theorem}
\begin{proof}
We first observe that the updates in \cref{eq:coupled-xy,eq:coupled-wo} can be written as a four-player ``best‐response'' game:
\begin{equation*}
\begin{aligned}
\widehat{x}_{t}&=\arg\min\nolimits_{x\in X} \eta_t\, \boldsymbol{w}_t^\mathrm{T}
\boldsymbol{\mathit{\Lambda}}_{t}(x,\widehat{y}_{t})\,
\boldsymbol{\omega}_t
+ w_t B_{\phi}\big(x, \widetilde{x}_t^{\phi}\big),  \\
\widehat{y}_{t}&=\arg\min\nolimits_{y\in Y} -\gamma_t\, \boldsymbol{w}_t^\mathrm{T}
\boldsymbol{\mathit{\Lambda}}_{t}(\widehat{x}_{t},y)\,
\boldsymbol{\omega}_t
+ \omega_t B_{\psi}\big(y, \widetilde{y}_t^{\psi}\big), \\
w_{t}&=\arg\min\nolimits_{w\in[T^{-1},1-T^{-1}],\boldsymbol{w}=[w,1-w]^\top} \theta_t\, \boldsymbol{w}^\mathrm{T}
\boldsymbol{\mathit{\Lambda}}_{t}
\boldsymbol{\omega}_{t} +\KL (\boldsymbol{w},\widetilde{\boldsymbol{w}}_t), \\
\omega_{t}&=\arg\min\nolimits_{\omega\in[T^{-1},1-T^{-1}],\boldsymbol{\omega}=[\omega,1-\omega]^\top} -\vartheta_t\, \boldsymbol{w}_{t}^\mathrm{T}
\boldsymbol{\mathit{\Lambda}}_{t}
\boldsymbol{\omega} +\KL (\boldsymbol{\omega},\widetilde{\boldsymbol{\omega}}_t), 
\end{aligned}
\end{equation*}
where $\boldsymbol{w}_t = [w_t, 1-w_t]^\mathrm{T}$, $\boldsymbol{\omega}_t = [\omega_t, 1-\omega_t]^\mathrm{T}$. 
Next, define the joint decision vector $\boldsymbol{x}=[x,y,w,\omega]^\top\in K$, where $K=X\times Y\times [T^{-1},1-T^{-1}]\times[T^{-1},1-T^{-1}]$ is compact convex, and define the operator
\begin{equation}
\label{eq:monotone-G}
\begin{aligned}
\boldsymbol{G}(\boldsymbol{x})
=\begin{bmatrix}
\nabla_x \ell_1(\boldsymbol{x})\\[2pt]
\nabla_y \ell_2(\boldsymbol{x})\\[2pt]
\nabla_w \ell_3(\boldsymbol{x})\\[2pt]
\nabla_\omega \ell_4(\boldsymbol{x})
\end{bmatrix},
\end{aligned}
\qquad\text{where}\quad
\begin{aligned}
\ell_1(\boldsymbol{x}) &= \eta_t\, \boldsymbol{w}^\mathrm{T}
\boldsymbol{\mathit{\Lambda}}_{t}(x,y)\,
\boldsymbol{\omega} /w + B_{\phi}\big(x, \widetilde{x}_t^{\phi}\big),\\
\ell_2(\boldsymbol{x}) &= -\gamma_t\, \boldsymbol{w}^\mathrm{T}
\boldsymbol{\mathit{\Lambda}}_{t}(x,y)\,
\boldsymbol{\omega}/\omega + B_{\psi}\big(y, \widetilde{y}_t^{\psi}\big),\\
\ell_3(\boldsymbol{x}) &= \theta_t\, \boldsymbol{w}^\mathrm{T}
\boldsymbol{\mathit{\Lambda}}_{t}(x,y)\,
\boldsymbol{\omega} +\KL (\boldsymbol{w},\widetilde{\boldsymbol{w}}_t),\\
\ell_4(\boldsymbol{x}) &= -\vartheta_t\, \boldsymbol{w}^\mathrm{T}
\boldsymbol{\mathit{\Lambda}}_{t}(x,y)\,
\boldsymbol{\omega} +\KL (\boldsymbol{\omega},\widetilde{\boldsymbol{\omega}}_t),
\end{aligned}
\end{equation}
with $\boldsymbol{w} = [w, 1-w]^\mathrm{T}$ and $\boldsymbol{\omega} = [\omega, 1-\omega]^\mathrm{T}$. 
By construction, each $\ell_i$ is 1-strongly convex in its own coordinate, so $\boldsymbol{G}$ is 1-strongly monotone with respect to the norm $\lVert\boldsymbol{x}\rVert^2=\lVert x\rVert^2+\lVert y\rVert^2+w^2+\omega^2$. 
Hence, the Browder-Minty theorem~\citep{brezis2010functional} guarantees a \emph{unique} point $\boldsymbol{x}^*\in K$ satisfying the variational inequality (VI):
\[
\left\langle \boldsymbol{G}(\boldsymbol{x}^*),\,\boldsymbol{x}-\boldsymbol{x}^*\right\rangle \geq 0,
\qquad\forall \boldsymbol{x}\in K.
\]
By the block-structure of $\boldsymbol{G}$, this $\boldsymbol{x}^*=[\widehat x_t,\widehat y_t,w_t,\omega_t]^\top$ coincides with the unique Nash equilibrium of the ``best-response'' game, and thus solves \cref{eq:coupled-xy,eq:coupled-wo}. 
\end{proof}

\begin{remark}
Browder-Minty Theorem~\citep{brezis2010functional} : 
Let $K$ be a nonempty compact convex set, and let $\boldsymbol{G}$ be continuous and $\mu$-strongly monotone defined on $K$, i.e., 
\[
\left\langle \boldsymbol{G}(\boldsymbol{x})-\boldsymbol{G}(\boldsymbol{x}'),\,\boldsymbol{x}-\boldsymbol{x}'\right\rangle\geq\mu\,\left\lVert \boldsymbol{x}-\boldsymbol{x}'\right\rVert^2,\qquad
\forall \boldsymbol{x},\boldsymbol{x}'\in K,
\]
for some $\mu>0$. 
Then there exists a \emph{unique} point $\boldsymbol{x}^*\in K$ satisfying the variational inequality
\[
\left\langle \boldsymbol{G}(\boldsymbol{x}^*),\,\boldsymbol{x}-\boldsymbol{x}^*\right\rangle \geq 0,
\qquad\forall \boldsymbol{x}\in K.
\]
\end{remark}

Now finding the solution to \cref{eq:coupled-xy,eq:coupled-wo} reduces to identifying a point $\boldsymbol{x}^*$ that satisfies the corresponding VI. Since the predictor $h_t$ is under our control, we assume it has a Lipschitz-continuous gradient (see Assumption~\ref{ass:Lipschitz-continuous-gradient}), which in turn ensures that the operator $\boldsymbol{G}$ is Lipschitz continuous (see Proposition~\ref{prop:G-Lipschitz}). Under Assumption~\ref{ass:Lipschitz-continuous-gradient}, we can invoke \cref{alg:joint-solver} --- originally proposed by \citet{Nesterov2006solving} --- to approximate $\boldsymbol{x}^*$. This method guarantees a global linear convergence rate (refer to Proposition~\ref{prop:convergence}).

\begin{assumption}
\label{ass:Lipschitz-continuous-gradient}
All predictors have Lipschitz‐continuous gradients. Specifically, there exist finite constants $L_{xx}$, $L_{xy}$, $L_{yx}$, and $L_{yy}$, such that $\forall x,x'\in X$, $\forall y,y'\in Y$, and $\forall t$: 
\begin{equation*}
\begin{aligned}
\lVert \nabla_x h_t(x,y)-\nabla_x h_t(x',y')\rVert&\leq L_{xx}\left\lVert x-x'\right\rVert+L_{xy}\left\lVert y-y'\right\rVert, \\
\lVert \nabla_y (-h_t)(x,y)-\nabla_y (-h_t)(x',y')\rVert&\leq L_{yx}\left\lVert x-x'\right\rVert+L_{yy}\left\lVert y-y'\right\rVert.
\end{aligned}
\end{equation*}
\end{assumption}

\begin{proposition}
\label{prop:G-Lipschitz}
Under Assumption~\ref{ass:Lipschitz-continuous-gradient}, $\boldsymbol{G}$ is Lipschitz continuous, that is, 
\begin{equation*}
\begin{aligned}
\left\lVert\boldsymbol{G}(\boldsymbol{x}) - \boldsymbol{G}(\boldsymbol{x}')\right\rVert
\leq L\left\lVert \boldsymbol{x}-\boldsymbol{x}'\right\rVert,\qquad \forall \boldsymbol{x},\boldsymbol{x}'\in K, 
\end{aligned}
\end{equation*}
where the Lipschitz constant $L$ is given by
\begin{equation}
\label{eq:Lipschitz}
\begin{aligned}
L = \sqrt{\max\{C_x,C_y,C_w,C_\omega\}},
\end{aligned}
\end{equation}
with $C_x = 4 \bigl((\eta_t L_{xx}+L_\phi)^2 + \gamma_t^2 L_{yx}^2 + (\theta_t^2 + 4\,\vartheta_t^2)\,G_X^2\bigr)$, $C_y = 4 \bigl((\gamma_t L_{yy}+L_\psi)^2 + \theta_t^2 L_{xy}^2 + (\vartheta_t^2 + 4\,\theta_t^2)\,G_Y^2\bigr)$, $C_w = 2\,\gamma_t^2 L_{yx}^2 D_X^2 + 4\,\vartheta_t^2 C$, $C_\omega = 2\,\eta_t^2 L_{xy}^2 D_Y^2 + 4 \,\theta_t^2 C$, and $C = \min\big\{D_X^2 (L_{xx}D_X+L_{xy}D_Y)^2,\,D_Y^2 (L_{yx}D_X+L_{yy}D_Y)^2\big\}+T^2$. 
\end{proposition}

\begin{proposition}
\label{prop:convergence}
Let $\boldsymbol{x}^* = [\widehat x_t,\widehat y_t,w_t,\omega_t]^\top$ be the solution to \cref{eq:coupled-xy,eq:coupled-wo}. 
Suppose \cref{alg:joint-solver} has performed $k$ rounds of iterations. 
Then its output satisfies
\begin{equation*}
\begin{aligned}
\left\lVert\boldsymbol{x}^* - \left(\sum\nolimits_{i=0}^k \lambda_i\right)^{-1}\sum_{i=0}^k \lambda_i \boldsymbol{y}_i \right\rVert
\leq \lVert\boldsymbol{G}(\boldsymbol{y}_0)\rVert\left(\frac{L}{L+1}\right)^{k/2}.
\end{aligned}
\end{equation*}
\end{proposition}

The proof of Proposition~\ref{prop:G-Lipschitz} is provided in Appendix~\ref{app:SuppProofs}. Proposition~\ref{prop:convergence} follows directly from \citet{Nesterov2006solving} by setting the strong‐monotonicity constant $\mu=1$.

%
\begin{algorithm*}[t]
\caption{Solving \cref{eq:coupled-xy,eq:coupled-wo}}
\label{alg:joint-solver}
\begin{algorithmic}[1]
\State \textbf{Require:} $X$ and $Y$ satisfy Assumption~\ref{ass:X-Y-bounded}. Predictor $h_t$ satisfies Assumption~\ref{ass:Lipschitz-continuous-gradient}
\State \textbf{Initialize:} $\boldsymbol{y}_{0}\in K=X\times Y\times [T^{-1},1-T^{-1}]\times[T^{-1},1-T^{-1}]$, $\lambda_0 = 1$, $k =  0$
\State Calculate $\boldsymbol{G}(\boldsymbol{x})$ using \cref{eq:monotone-G} or \cref{eq:monotone-G-detail}, and set $L$ via \cref{eq:Lipschitz}
\Repeat
\State Update $\boldsymbol{x}_k = \arg\min_{\boldsymbol{x} \in K} \sum_{i=0}^k \lambda_i \bigl(\left\langle \boldsymbol{G}(\boldsymbol{y}_i), \boldsymbol{x}\right\rangle + \left\lVert \boldsymbol{y}_i-\boldsymbol{x} \right\rVert^2/2\bigr)$
\State Update $\boldsymbol{y}_{k+1} = \arg\min_{\boldsymbol{x} \in K}\, \langle \boldsymbol{G}(\boldsymbol{x}_k), \boldsymbol{x} \rangle + L \lVert \boldsymbol{x}-\boldsymbol{x}_k\rVert^2/2$
\State Update $\lambda_{k+1} = \frac{1}{L} \sum_{i=0}^k \lambda_i$
\State $k\gets k+1$
\Until{$\bigl(\frac{L}{L+1}\bigr)^{k/2}\downarrow 0$}
\State \textbf{Output:} $[\widehat x_t,\widehat y_t,w_t,\omega_t]^\top\gets\bigl(\sum_{i=0}^k \lambda_i\bigr)^{-1}\sum_{i=0}^k \lambda_i \boldsymbol{y}_i$
\end{algorithmic}
\end{algorithm*}

\subsection{Multi-Predictor Aggregator}

The output of the integrated module achieves minimax optimality and effectively reduces the D-DGap when using an accurate predictor sequence. However, relying on a single predictor sequence limits the algorithm's adaptability to different environments. To address this, we consider having $d$ available predictor sequences, each potentially derived from distinct models of the underlying environment. Our goal is to enhance the Integration Module by supporting multiple predictors, enabling it to retain minimax optimality while dynamically adapting to the most effective predictor sequence across these models. 

The Hedge algorithm is a well-established no-regret algorithm, known for its ability to perform consistently close to the best expert's strategy over time. This property makes it particularly effective in scenarios involving multiple predictors. Building on this foundation, we designed the \emph{Multi-Predictor Aggregator} using the Hedge algorithm. 

At each round $t$, the aggregator takes $d$ available predictors, denoted as $\{h_t^1, h_t^2, \cdots, h_t^d\}$, and outputs a combined predictor $h_t = \sum_{k=1}^d \xi_t^k h_t^k$ to the Integration Module. Here, the weight vector $\boldsymbol{\xi}_t = [\xi_t^1, \xi_t^2, \cdots, \xi_t^d]^\mathrm{T}$ is computed using the clipped Hedge algorithm. Specifically, the weights are updated by solving the following optimization problem:  
\begin{equation}
\label{eq:hedge}
\begin{aligned}
\boldsymbol{\xi}_{t+1} = \arg\min_{\boldsymbol{\xi} \in \bigtriangleup_d^a} \zeta_t \left\langle \boldsymbol{L}_t, \boldsymbol{\xi} \right\rangle + \KL(\boldsymbol{\xi}, \boldsymbol{\xi}_t),
\end{aligned}
\end{equation}
where $a=d/T$, $\zeta_t$ is the learning rate, and $\boldsymbol{L}_t$ denotes the loss vector: 
\begin{equation*}
\begin{aligned}
\boldsymbol{L}_{t} = [L_t^1, L_t^2, \cdots, L_t^d]^\mathrm{T},\qquad L_t^k = \max_{x\in\{\widehat{x}_t,\overline{x}_t,\widetilde{x}_{t+1}\},y\in\{\widehat{y}_t,\overline{y}_t,\widetilde{y}_{t+1}\}} \left\lvert f_t(x,y) - h_t^k(x,y)\right\rvert. 
\end{aligned}
\end{equation*}

The following theorem states that the Multi-Predictor Aggregator effectively provides multiple predictor support for the Integration Module. 
%
\begin{theorem}[D-DGap for the Integration Module with a Multi-Predictor Aggregator]
\label{thm:overall}
Assume the payoff function $f_t$ and all predictors $\{h_t^1, h_t^2, \cdots, h_t^d\}$ satisfy Assumption~\ref{ass:2-subgradient-bounded}. Let $T \geq d$. If the Multi-Predictor Aggregator updates its learning rate according to the following equations:
\begin{equation*}
\begin{aligned}
\textstyle \zeta_t = (\ln T)\big/\bigl(\epsilon + \sum_{\tau=1}^{t-1}\Delta_{\tau}\bigr), 
\qquad\epsilon >0,\qquad
0\leq\Delta_{t} = \left\langle \boldsymbol{L}_{t}, \boldsymbol{\mathit{\xi}}_t-\boldsymbol{\mathit{\xi}}_{t+1}\right\rangle -\KL (\boldsymbol{\mathit{\xi}}_{t+1},\boldsymbol{\mathit{\xi}}_{t})/\zeta_t. 
\end{aligned}
\end{equation*}
Then, the D-DGap upper bound for the Integration Module can be enhanced as follows: 
\begin{equation*}
\begin{aligned}
\ddgap\,(u_{1:T},v_{1:T})\leq \widetilde{O}\left(\min\left\{\min_{k \in \{1,2,\cdots,d\}}\sum_{t=1}^T \rho(f_t, h_t^k),\ \sqrt{\left(1 + \min\{P_T, C_T\}\right)T}\right\}\right).
\end{aligned}
\end{equation*}
\end{theorem}
%
%
The clipped Hedge equivalent to the following update:
\begin{equation*}
\begin{aligned}
\boldsymbol{\xi}_{t+1}=\arg\min_{\boldsymbol{\xi} \in \bigtriangleup_d^a}\left\langle\ln \frac{\boldsymbol{\xi}}{\boldsymbol{\xi}_{t}\cdot\exp(-\zeta_{t}\boldsymbol{L}_t)},\ \boldsymbol{\xi}\right\rangle,
\end{aligned}
\end{equation*}
Thus, an efficient solution is attainable by minor adjustments to the algorithm depicted in Figure~3 of \citet{herbster2001tracking}. 

\subsection{Structure and Advantages}

In the previous sections, we analyzed the Adaptive Module, Integration Module, and Multi-Predictor Aggregator individually. To clarify how these modules work together to form the overall algorithm, we present a structural detail~(see \cref{fig:Diagram}) and accompanying pseudocode~(see \cref{alg:pseudocode}).
\begin{figure}[tb]
\centering
\resizebox{1.0\linewidth}{!}{
\centering\footnotesize
\tikzstyle{block} = [draw, rectangle, fill=blue!4, rounded corners, minimum height=2em, minimum width=10em, align=center]
\tikzstyle{bigblock} = [draw, rectangle, fill=blue!4, rounded corners, minimum height=3em, minimum width=15em, align=center]
\tikzstyle{sum} = [draw, circle, fill=none, minimum height=1em, minimum width=1em, align=center,inner sep=1.6pt]
\tikzstyle{feedback} = [draw=none, fill=none,minimum height=1em, minimum width=6em, align=center]
\tikzstyle{output} = [draw,rectangle,rounded corners, fill=red!4,minimum height=4em, minimum width=6em, align=center]
\tikzstyle{input} = [draw=none, fill=none,minimum height=1em, minimum width=1em, align=right]
\tikzstyle{point} = [coordinate]
\hspace*{-2.3em}
\begin{tikzpicture}[auto,>=latex',line/.style={-Stealth,thick}]
\node [bigblock] (algorithm) {\cref{eq:coupled-xy,eq:x-tilde,eq:y-tilde,eq:coupled-wo,eq:w-tilde,eq:o-tilde}, with \\\cref{eq:coupled-xy,eq:coupled-wo} solved via \cref{alg:joint-solver}};
\node [block, above left =2.7em and -9.3em of algorithm] (aderx) {ADER or ADER-like algorithm};
\node [block, above right=2.7em and -9.3em of algorithm] (adery) {ADER or ADER-like algorithm};
\node [point, above=1.7em of algorithm] (aderxy) {};
\draw (aderx) |- node[midway, left ] {$\overline{x}_t$} (aderxy);
\draw (adery) |- node[midway, right] {$\overline{y}_t$} (aderxy);
\draw [line] (aderxy) -- node[midway, right] {$(\overline{x}_t,\overline{y}_t)$} (algorithm);
\node [block, below=2em of algorithm] (predictor) {\cref{eq:hedge}};
\draw [-Stealth] (predictor) -- node[midway, right] {$h_t = \sum_{k=1}^d \xi_t^k h_t^k$} (algorithm);
\node [input, right=2.5em of predictor] (input) {Predictors $\{h_t^1, h_t^2,\cdots,h_t^d\}$};
\draw [line] (input) -- (predictor);

\node [sum, right=9em of algorithm] (sum) {$\sum$};
\node [point, left=1.8em of sum] (fakesum) {};
\draw [thick] ($(algorithm.0) + (0,0.7em)$)  -- node[midway, above] {$(\widehat{x}_t,\widehat{y}_t)$, $(\overline{x}_t,\overline{y}_t)$} ($(fakesum.90) + (0,0.7em)$);
\draw [line] ($(fakesum.180) + (0,0.7em)$) -- (sum);
\draw [thick] ($(algorithm.0) + (0,-0.7em)$)  --  node[midway, below] {$(\boldsymbol{w}_t,\boldsymbol{\omega}_t)$} ($(fakesum.270) + (0,-0.7em)$);
\draw [line] ($(fakesum.180) + (0,-0.7em)$) -- (sum);
\node [output, right=7.5em of sum] (output) {Environment};
\draw [line] (sum) -- node[midway, above] {$x_t = [\widehat{x}_{t}, \overline{x}_t]\boldsymbol{w}_t$}  node[midway, below] {$y_t = [\widehat{y}_{t}, \overline{y}_t]\boldsymbol{\omega}_t$} (output);
\node [point, above=1.8em of aderx] (feed0) {};
\node [point, above=1.8em of adery] (feed1) {};
\draw [densely dotted,thick] (output) |- node[midway, below left] {Feedback} (feed0);
\draw [densely dotted,line] (feed0) -- node[midway, left] {$f_t(\,\cdot\,,y_t)$} (aderx);
\draw [densely dotted,line] (feed1) -- node[midway, right] {$-f_t(x_t,\,\cdot\,)$} (adery);
\node [point, below=1.8em of predictor] (feed3) {};
\draw [densely dotted,thick] (output) |- node[midway, above left] {Feedback} (feed3);
\draw [densely dotted,line] (feed3) -- node[midway, right] {$\boldsymbol{L}_t$} (predictor);
\node [point, left=3em of algorithm] (feed2) {};
\draw [densely dotted,thick] (feed3) -| (feed2);
\draw [densely dotted,line] (feed2) -- node[midway, below] {$f_t$} (algorithm);
\node [point, left=4.3em of algorithm] (IM) {};
\node [input, left=1em of IM] (im) {$\begin{aligned}\textup{Integration}\\[-2pt]\textup{Module}\end{aligned}\Bigg\{$};
\node [point] (AM) at (aderx -| IM) {};
\node [input, left=1em of AM] (am) {$\begin{aligned}\textup{Adaptive}\\[-2pt]\textup{Module}\end{aligned}\Bigg\{$};
\node [point] (MP) at (predictor -| IM) {};
\node [input, left=1em of MP] (mp) {$\begin{aligned}\textup{Multi-Predictor}\\[-2pt]\textup{Aggregator}\end{aligned}\Bigg\{$};
\end{tikzpicture}
}
\caption{Structural Detail of Our Modular Algorithm.}
\label{fig:Diagram}
\end{figure}
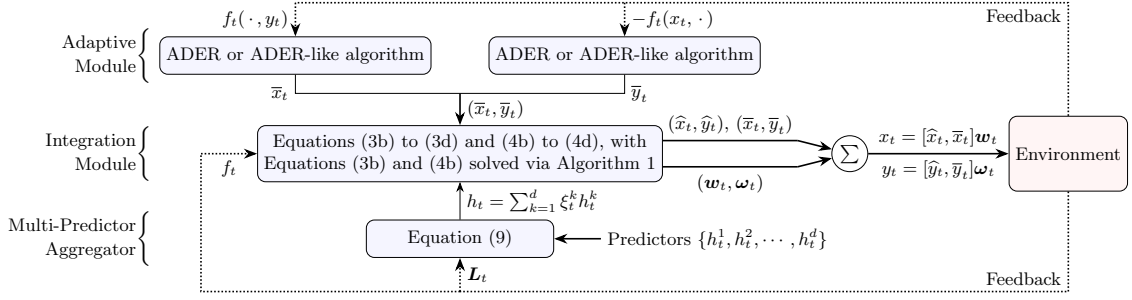
\begin{algorithm*}[t]
\caption{Pseudocode for Our Modular Algorithm}
\label{alg:pseudocode}
\begin{algorithmic}[1]
\State \textbf{Require:} $X$ and $Y$ satisfy Assumption~\ref{ass:X-Y-bounded}. All payoff functions satisfy Assumption~\ref{ass:2-subgradient-bounded}.  All predictors satisfy Assumptions~\ref{ass:2-subgradient-bounded} and~\ref{ass:Lipschitz-continuous-gradient}
\State \textbf{Initialize:} $\widetilde{x}_{1}$, $\widetilde{y}_{1}$, $\widetilde{\boldsymbol{w}}_{1}$, $\widetilde{\boldsymbol{\omega}}_{1}$, $\boldsymbol{\xi}_{1}$ and $(\overline{x}_1, \overline{y}_1)$
\For{$t\leftarrow1$ \textbf{to} $T$}
\State Receive $d$ predictors $h_t^1, h_t^2, \cdots, h_t^d$ and compute $h_t = \sum_{k=1}^d \xi_t^k h_t^k$

\State Obtain $(\widehat{x}_{t},\widehat{y}_{t})$ and $(\boldsymbol{w}_t,\boldsymbol{\omega}_t)$ via \cref{alg:joint-solver}

\State Output $x_t = [\widehat{x}_{t}, \overline{x}_t]\boldsymbol{w}_t$, $y_t = [\widehat{y}_{t}, \overline{y}_t]\boldsymbol{\omega}_t$, and then observe $f_t$

\State Update $\widetilde{x}_{t+1}$, $\widetilde{y}_{t+1}$, $\widetilde{\boldsymbol{w}}_{t+1}$ and $\widetilde{\boldsymbol{\omega}}_{t+1}$ using \cref{eq:x-tilde,eq:y-tilde,eq:w-tilde,eq:o-tilde}

\State Update $\boldsymbol{\xi}_{t+1}$ according to \cref{eq:hedge}

\State Update $(\overline{x}_{t+1}, \overline{y}_{t+1})$ by running two ADER or ADER-like algorithms
\EndFor
\end{algorithmic}
\end{algorithm*}

\cref{thm:overall} provides the D-DGap upper bound guarantee for the entire algorithm, which can be rearranged as follows:
\begin{equation}
\label{eq:ddgap-bound}\setlength{\belowdisplayskip}{6.5pt}
\begin{aligned}
\ddgap\,(u_{1:T},v_{1:T})\leq\widetilde{O}\Bigl(\min\Bigl\{ \underbrace{\min\bigl\{V^1_T,\ \cdots,\ V^d_T\bigr\}}_{\term{a}},\ \underbrace{\sqrt{\left(1 + \min\{P_T, C_T\}\right)T}}_{\term{b}}\Bigr\}\Bigr),
\end{aligned}
\end{equation}
where $V^k_T=\sum_{t=1}^T \rho(f_t, h_t^k)$ represents the cumulative prediction error of the $k$-th predictor.

The Adaptive Module ensures a minimax-optimal bound, as given by \termref{eq:ddgap-bound}{b}, allowing the algorithm to adapt to varying levels of non-stationarity. The Multi-Predictor Aggregator provides the bound in \termref{eq:ddgap-bound}{a}, ensuring that if any one predictor models the environment well, the algorithm achieves a sharp $\widetilde{O}(1)$ D-DGap. This acts as an automatic selection mechanism for the best predictor. 

The Integration Module combines both components, ensuring adaptability to dynamic environments while effectively tracking the best predictor. It guarantees near-optimal performance across different settings, with any further improvement limited to at most a logarithmic factor. The modular design allows for easy replacement of components that regulate adaptivity and the integration of ``side knowledge'' from other predictors. 

Unlike the Multi-Predictor Aggregator, which applies to both OCCO and OCO, the Integration Module's interdependent update mechanism is specific to OCCO. This is because decomposing the D-DGap in OCCO requires a more intricate approach, as demonstrated in the proof of \cref{thm:Integration}. In contrast, in OCO, D-Reg can be directly decomposed into the meta-layer regret and the individual regret of any expert. Thus, for OCO, it suffices to add an extra expert in ADER to obtain a prediction error-based upper bound while replacing the meta-layer algorithm with optimistic clipped Hedge.

\section{Experiments}
\label{app:Experiments}

This section experimentally validates the effectiveness of our algorithm, comparing it against the algorithm proposed by \cite{zhang2022noregret} and a pair of ADERs as benchmarks.

We consider a specific instance of the OCCO problem, where the feasible domain is defined as $X \times Y = [-1, 1]^2$, and the environment provides the following convex-concave payoff function at round $t$:
\begin{equation}
\begin{aligned}
f_t\left(x,y\right)=\frac{1}{2}\left(x-x_t^*\right)^2-\frac{1}{2}\left(y-y_t^*\right)^2+\left(x-x_t^*\right)\left(y-y_t^*\right), 
\end{aligned}
\end{equation}
where $(x_t^*, y_t^*) \in X \times Y$ denotes the saddle point of $f_t$. 
This setup satisfies Assumptions~\ref{ass:X-Y-bounded} and~\ref{ass:2-subgradient-bounded}. The evolution of the saddle point $(x_t^*, y_t^*)$ reflects specific environmental characteristics. We identify four distinct cases, as outlined in \cref{tab:rules}: 
\begin{itemize}[itemsep=0pt,topsep=4pt]
\item Case~\hyperlink{ruleI}{I} indicates a gradually stationary environment, with the movement of the saddle point diminishing over time.
\item Case~\hyperlink{ruleII}{II} and~\hyperlink{ruleIII}{III} represent approximate periodic environments. In Case~\hyperlink{ruleII}{II}, the saddle point cycles among three branches, while in Case~\hyperlink{ruleIII}{III}, it cycles among seven, with its position in each branch chosen randomly.
\item Case~\hyperlink{ruleIV}{IV} depicts an adversarial environment where the saddle point cannot be effectively approximated. In this case, upon selecting a strategy pair $(x_t, y_t)$, the environment generates the saddle point $(x_t^*, y_t^*)$ by rotating the strategy pair by a random angle $\varphi \sim N(\pi, 1)$ and then projecting it onto the circle of radius $1/2$.
\end{itemize}
\begin{table}[t]
\centering\setlength{\tabcolsep}{0.3em}
\caption{Four Environment Settings. In this table, the saddle point $(x_t^*, y_t^*)$ is expressed in the complex form $p_t^*=x_t^* + i y_t^*$, where $i$ is the imaginary unit, satisfying $i^2 = -1$. $z_1(t) = \ln(1+t)$, $z_2(t) = \ln\ln(\mathrm{e} + t)$. As $t$ increases, the growth rates of both $z_1$ and $z_2$ gradually decelerate. $\varepsilon\sim U(0, 1)$ is random variable that follows a uniform distribution on the interval $[0,1]$, and $\varphi \sim N(\pi, 1)$ is a random angle that follows a Gaussian distribution with mean $\pi$. }
\label{tab:rules}
\begin{tabular}{ccccc}
\toprule
Case & \hypertarget{ruleI}{I} & \hypertarget{ruleII}{II} & \hypertarget{ruleIII}{III} & \hypertarget{ruleIV}{IV} \\
\midrule
$x_t^*+i y_t^*$ & $\frac{1}{3}z_2(t)\mathrm{e}^{i z_1(t)}$ & $\frac{1}{3}z_2(t)\mathrm{e}^{i\frac{2\pi}{3} t+i z_2(t)}$ & $\frac{1}{2}\mathrm{e}^{\frac{1}{7}(\varepsilon+i2\pi t)}$
& $\frac{1}{2}\mathrm{e}^{i\left(\varphi+\arg(x_t+i y_t)\right)}$ \\[2pt]
Trajectories & 
\begin{tabular}{c}
\includegraphics[width=0.18\textwidth]{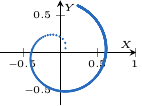}
\end{tabular} &
\begin{tabular}{c}
\includegraphics[width=0.18\textwidth]{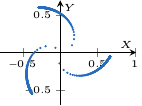}
\end{tabular} &
\begin{tabular}{c}
\includegraphics[width=0.18\textwidth]{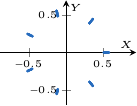}
\end{tabular} &
\begin{tabular}{c}
\includegraphics[width=0.18\textwidth]{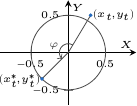}
\end{tabular} \\[-2pt]
Property & $\rho(f_t,f_{t-1})\rightarrow 0$ & $\rho(f_t,f_{t-3})\rightarrow 0$ & $\left\lvert p_t^*-p_{t-7}^*\right\rvert\leq\frac{1-\mathrm{e}^{\sfrac{1}{7}}}{2}$ & Adversarial \\
\bottomrule
\end{tabular}
\end{table}

To capture a range of non-stationarity levels, we select three comparator sequences representing different dynamics, from fully stationary to highly non-stationary settings, as detailed in \cref{tab:levels}. 
%
\begin{table}[t]
\centering\setlength{\tabcolsep}{0.9em}
\caption{Three Levels on Comparator Sequence Non-Stationarity.  In this table, $x'_t = \arg\min_{x \in X} f_t(x, y_t)$, and $y'_t = \arg\max_{y \in Y} f_t(x_t, y)$.}
\label{tab:levels}
\begin{tabular}{cccc}
\toprule
Level & \hypertarget{leveli}{i} & \hypertarget{levelii}{ii} & \hypertarget{leveliii}{iii} \\
\midrule
Comparator & 
$(u_t,v_t) \equiv (0,0)$ & 
$(u_t,v_t) = (x_t^*,y_t^*)/\ln(1+t)$ & 
$(u_t,v_t) = (x'_t,y'_t)$ \\
\bottomrule
\end{tabular}
\end{table}

We instantiate our algorithm as follows: Let $\phi(x)=x^2/2$ and $\psi(y)=y^2/2$. Both $B_\phi$ and $B_\psi$ are bounded, 1-strongly convex, and exhibit Lipschitz continuity with respect to their first variables. For the Multi-Predictor Aggregator, we configure four predictors: $h_t^1 = f_{t-1}$, $h_t^2 = f_{t-3}$, $h_t^3 = f_{t-7}$, and $h_t^4 = f_{t-8}$, all of which satisfy Assumption~\ref{ass:Lipschitz-continuous-gradient}. This setup enables our algorithm to achieve a sharp D-DGap bound of $\widetilde{O}(1)$ in stationary environments or periodic scenarios with cycles of 2, 3, 4, 7, or 8. In the Integration Module, we employ Successive Reduction of Search Space for joint updates, maintaining computational costs within acceptable limits. We also apply the doubling trick~\cite{schapire1995gambling} to eliminate the algorithms' dependence on the time horizon $T$. 

\begin{figure}[t]
\centering
\includegraphics[width=\textwidth]{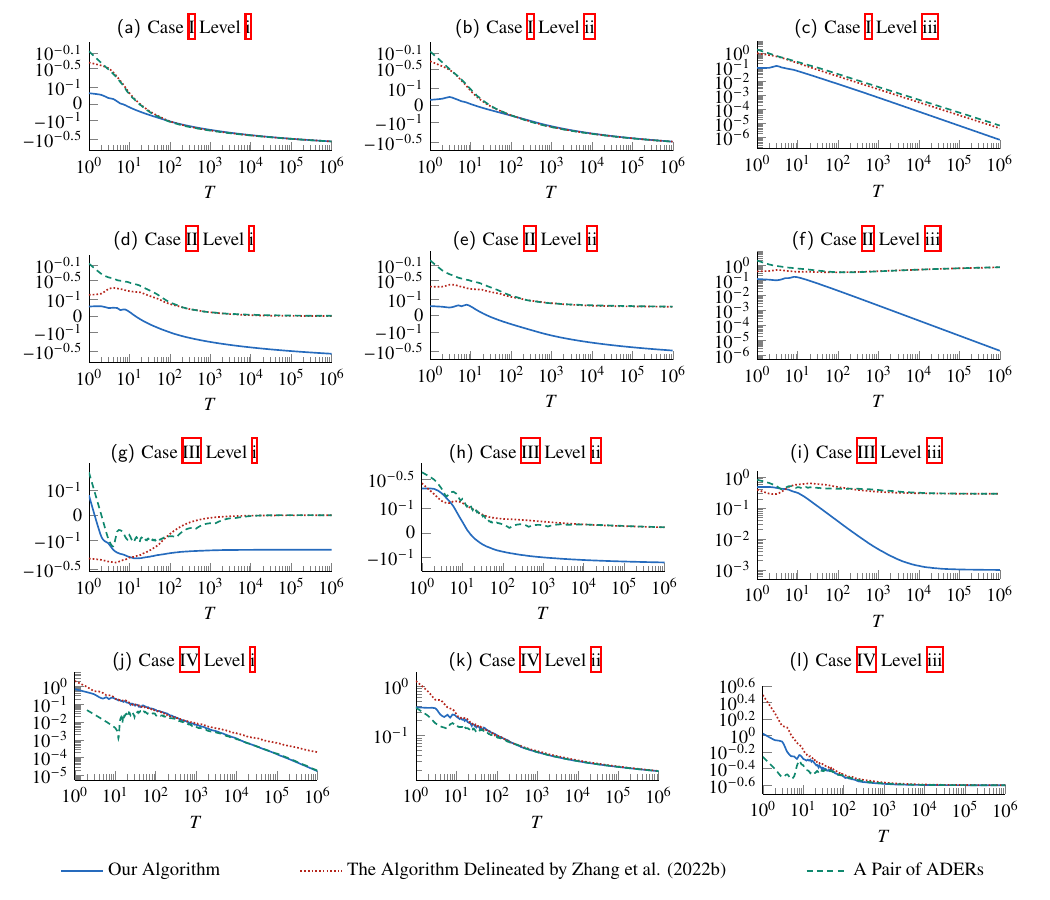}
\vspace*{-0.5em}
\caption{Time-Averaged D-DGaps of Algorithms}
\label{fig:results}
\end{figure}

We conduct $10^6$ rounds for each case and record the time-averaged D-DGap. The results in \cref{fig:results} align with theoretical expectations. 
In Case~\hyperlink{ruleI}{I} Level iii (refer to Figure~2c), our algorithm demonstrates better performance, as it progressively approaches $\widetilde{O}(1)$ D-DGap, while the other two algorithms converge towards $\widetilde{O}(\sqrt{T})$ D-DGaps. 
In Cases~\hyperlink{ruleII}{II} and~\hyperlink{ruleIII}{III}, our algorithm consistently outperforms both \cite{zhang2022noregret} and ADER algorithms. Notably, in Figure~2f, our algorithm successfully converges, whereas the other two fail to do so. 
In Case~\hyperlink{ruleIV}{IV}, all algorithms perform comparably. Both our algorithm and ADERs guarantee minimax optimality, while the algorithm in \cite{zhang2022noregret}, despite lacking tight bounds, shows empirical success due to the meta-expert framework.

\section{Conclusion}

This paper is the first to study the dynamic duality gap~(D-DGap) in Online Convex-Concave Optimization (OCCO). Our modular algorithmic structure adapts seamlessly to varying levels of non-stationarity and leverages the most accurate predictors, while the Integration Module, inspired by the meta-expert framework, ensures optimal performance across diverse environments.

A natural next step is to tackle the two-player, time-varying game, where the $x$–player observes only $f_t(\,\cdot\,,y_t)$ and the $y$–player only sees $-f_t(x_t,\,\cdot\,)$. This partial-observation model is weaker than our full-information setting, in which both players have access to the entire payoff function $f_t$. It raises two key challenges: (1) preserving our minimax-optimal D-DGap guarantee under one-sided feedback, and (2) Further tightening the D-DGap through more aggressive adaptation to each player's history. We plan to develop new algorithms that address these challenges while maintaining strong theoretical guarantees.


\acks{This work was supported by the National Natural Science Foundation of China~(Grant No. 62406299).}

\section*{CRediT author statement}
\begin{itemize}
\item Qing-xin Meng: Conceptualization; Methodology; Software; Formal analysis; Investigation;
Writing -- Original Draft; Visualization.
\item Xia Lei: Conceptualization; Validation; Writing -- Review \& Editing; Resources; Funding acquisition. 
\item Jian-wei Liu: Supervision; Project administration.
\end{itemize}



\appendix

\section{Supplementary Proofs}
\label{app:SuppProofs}
\subsection{Proof of Proposition~\ref{thm:lower-bound}}
\begin{proof}Let $\mathscr{F}$ denote all convex-concave functions satisfying Assumption~\ref{ass:2-subgradient-bounded}, and let $\mathscr{L}_X(G)=\big\{\ell\text{ is convex }\big|\,\sup_{x\in X}\left\lVert\partial\ell(x)\right\rVert\leq G\big\}$. 
The key to the proof is to convert OCCO into a pair of OCO problems: 
\begin{equation*}
\begin{aligned}
\sup_{f_1,\cdots,f_T\in\mathscr{F}}&\left(\sup_{P_T \leq P}\Bigl(f_t\left(x_t,v_t\right)- f_t\left(u_t,y_t\right)\Bigr)\right) \\
\geq\ & \sup_{f_t(x,y)=\alpha_t(x)-\beta_t(y)\in\mathscr{F},\,t\in\{1,\cdots,T\}}\left(\max_{P_T^u \leq p,\,P_T^v \leq P-p} \sum_{t=1}^T \Big(f_t\left(x_t,v_t\right)- f_t\left(u_t,y_t\right)\Big)\right) \\
=\ & \sup_{\alpha_1,\cdots,\alpha_T\in\mathscr{L}_X(G_X)}\left(\max_{P_T^u \leq p}\sum_{t=1}^T \Big(\alpha_t\left(x_t\right)- \alpha_t\left(u_t\right)\Big)\right) \\
&+\sup_{\beta_1,\cdots,\beta_T\in\mathscr{L}_Y(G_Y)}\left(\max_{P_T^v \leq P-p}\sum_{t=1}^T \Big(\beta_t\left(y_t\right)- \beta_t\left(v_t\right)\Big)\right)\\
\geq\ & \mathit{\Omega}\left(\sqrt{(1+p)T}\right)
+ \mathit{\Omega}\left(\sqrt{(1+P-p)T}\right),\qquad\qquad\forall 0\leq p\leq P,
\end{aligned}
\end{equation*}
where the first ``$\geq$'' follows from the specific structure of $f_t$, given by $f_t(x,y) = \alpha_t(x) - \beta_t(y)$. Here, both $\alpha_t$ and $\beta_t$ are convex functions, $P_T^u=\sum_{t=1}^T\lVert u_t-u_{t-1}\rVert$ and $P_T^v=\sum_{t=1}^T\lVert v_t-v_{t-1}\rVert$. The second ``$\geq$'' is derived from Theorem~2 in \citet{zhang2018adaptive}, which establishes a lower bound on regret for OCO. Combining these two lower bounds yields the desired result. 
\end{proof}

\subsection{Proof of Proposition~\ref{prop:Adaptive}}
\begin{proof}
Independently applying two ADER or ADER-like algorithms results in the following bounds:
\begin{equation*}
\begin{aligned}
\sum_{t=1}^T \Bigl(f_t(\overline{x}_t, y_t)-f_t(u_t, y_t)\Bigr)&\leq \widetilde{O}\left(\sqrt{\left(1+P_T^u\right)T}\right), \\
\sum_{t=1}^T \Bigl(f_t(x_t,v_t)-f_t(x_t, \overline{y}_t)\Bigr)&\leq \widetilde{O}\left(\sqrt{\left(1+P_T^v\right)T}\right), 
\end{aligned}
\end{equation*}
where $P_T^u = \sum\nolimits_{t=1}^T\lVert u_t-u_{t-1}\rVert$ and $P_T^v = \sum\nolimits_{t=1}^T\lVert v_t-v_{t-1}\rVert$. 
For specially chosen comparators $x'_t=\arg\min_{x\in X}f_t\left(x, y_t\right)$ and $y'_t=\arg\max_{y\in Y} f_t\left(x_t, y\right)$, we also have:
\begin{equation*}
\begin{aligned}
\sum_{t=1}^T \Bigl(f_t(\overline{x}_t, y_t)-f_t(u_t, y_t)\Bigr)&\leq\sum_{t=1}^T \Bigl(f_t(\overline{x}_t, y_t)-f_t(x'_t, y_t)\Bigr)\leq \widetilde{O}\left(\sqrt{\left(1+C_T^x\right)T}\right), \\
\sum_{t=1}^T \Bigl(f_t(x_t,v_t)-f_t(x_t, \overline{y}_t)\Bigr)&\leq\sum_{t=1}^T \Bigl(f_t(x_t,y'_t)-f_t(x_t, \overline{y}_t)\Bigr)\leq \widetilde{O}\left(\sqrt{\left(1+C_T^y\right)T}\right), 
\end{aligned}
\end{equation*}
where $C_T^x = \sum\nolimits_{t=1}^T\lVert x'_t-x'_{t-1}\rVert$ and $C_T^y = \sum\nolimits_{t=1}^T\lVert y'_t-y'_{t-1}\rVert$. \\
The desired result follows by combining these inequalities. 
\end{proof}

\subsection{Proof of Lemma~\ref{lem:OptOPPM}}
\begin{proof}
The first-order optimality condition of \cref{eq:optoppm} implies that 
\begin{equation*}
\begin{aligned}
&\exists\nabla_x h_t\big(x_t, y_t\big),&&\forall x'\in X,&&
\textstyle\bigl\langle \eta_t \nabla_x h_t\big(x_t, y_t\big) + x_t^{\phi}-\widetilde{x}_t^{\phi}, x_t-x'\bigr\rangle \leq 0, \\
&\exists\nabla_y (-h_t)\big(x_t, y_t\big),&&\forall y'\in Y,&&
\textstyle\bigl\langle \gamma_t \nabla_y (-h_t)\big(x_t, y_t\big) + y_t^{\psi}-\widetilde{y}_t^{\psi}, y_t-y'\bigr\rangle \leq 0, \\
&\exists\nabla_x f_t(\widetilde{x}_{t+1}, y_t),&&\forall x'\in X,&&
\bigl\langle \eta_t \nabla_x f_t(\widetilde{x}_{t+1}, y_t) + \widetilde{x}_{t+1}^{\phi}-\widetilde{x}_t^{\phi}, \widetilde{x}_{t+1}-x'\bigr\rangle \leq 0, \\
&\exists\nabla_y (-f_t)(x_{t},\widetilde{y}_{t+1}),&&\forall y'\in Y,&&
\textstyle\big\langle \gamma_t \nabla_y (-f_t)(x_{t},\widetilde{y}_{t+1})+\widetilde{y}_{t+1}^\psi-\widetilde{y}_{t}^\psi, \widetilde{y}_{t+1}-y'\big\rangle\leq 0.
\end{aligned}
\end{equation*}
The D-DGap is composed of the sum of two individual D-Regs:
\begin{equation*}
\begin{aligned}
\ddgap\,(u_{1:T},v_{1:T})&=\sum_{t=1}^T \Bigl(f_t\left(x_t, v_t\right)-f_t\left(u_t, y_t\right)\Bigr) \\
&=\sum_{t=1}^T \Bigl(f_t\left(x_t, v_t\right)-f_t\left(x_t, y_t\right)\Bigr) + \sum_{t=1}^T \Bigl(f_t\left(x_t, y_t\right)-f_t\left(u_t, y_t\right)\Bigr) \\
&=\dregret\,(v_{1:T})+\dregret\,(u_{1:T}).
\end{aligned}
\end{equation*}
Let's take the $x$-player as an example.
We first perform identity transformation on the instantaneous individual regret: 
\begin{equation}
\label{eqpf:identical-trans}
\begin{aligned}
f_t(x_t, y_t)-f_t(u_t, y_t)=\ &
\underbrace{f_t(x_t,y_t)-h_t(x_t,y_t)
+h_t(\widetilde{x}_{t+1},y_t)-f_t(\widetilde{x}_{t+1}, y_t)}_{\term{a}} \\
&+\underbrace{h_t(x_t,y_t)-h_t(\widetilde{x}_{t+1},y_t)
+f_t(\widetilde{x}_{t+1}, y_t)-f_t(u_t, y_t)}_{\term{b}}.
\end{aligned}
\end{equation}
By using convexity and first-order optimality conditions, we get
\begin{equation}
\label{eq:1-order-x-bound}
\begin{aligned}
\termref{eqpf:identical-trans}{b}
\leq\ &\big\langle \nabla_x h_t\big(x_t, y_t\big), x_t-\widetilde{x}_{t+1}\big\rangle + \big\langle\nabla_x f_t(\widetilde{x}_{t+1}, y_t), \widetilde{x}_{t+1}-u_t\big\rangle \\
\leq\ &\big\langle \widetilde{x}_{t}^\phi-x_{t}^\phi,x_t-\widetilde{x}_{t+1}\big\rangle/ \eta_t  + \big\langle \widetilde{x}_{t}^\phi-\widetilde{x}_{t+1}^\phi,\widetilde{x}_{t+1}-u_t\big\rangle / \eta_t\\
=\ &\big[B_{\phi}\big(\widetilde{x}_{t+1}, \widetilde{x}_t^\phi\big)-B_{\phi}\big(\widetilde{x}_{t+1}, x_t^\phi\big)-B_{\phi}\big(x_t, \widetilde{x}_t^\phi\big)\big]/ \eta_t \\
&+\underbrace{\big[B_{\phi}\big(u_t, \widetilde{x}_{t}^\phi\big)-B_{\phi}\big(u_t, \widetilde{x}_{t+1}^\phi\big)\big] / \eta_t}_{\eqcolon\Phi_t}-B_{\phi}\big(\widetilde{x}_{t+1}, \widetilde{x}_{t}^\phi\big) / \eta_t.
\end{aligned}
\end{equation}
Let $\nu_t^x=\termref{eqpf:identical-trans}{a}-B_{\phi}\big(\widetilde{x}_{t+1}, x_t^\phi\big)/\eta_t$, so we have that $\nu_t^x\geq 0$. To verify this, it suffices to combine the following two inequalities:
\begin{equation*}
\begin{aligned}
f_t(x_t,y_t)+B_{\phi}\big(x_t, \widetilde{x}_{t}^\phi\big)/\eta_t
&\geq f_t(\widetilde{x}_{t+1},y_t)+B_{\phi}\big(\widetilde{x}_{t+1}, \widetilde{x}_{t}^\phi\big)/\eta_t, \\
-h_t(x_t,y_t)+h_t(\widetilde{x}_{t+1},y_t)
&\geq-\big[B_{\phi}\big(\widetilde{x}_{t+1}, \widetilde{x}_t^\phi\big)-B_{\phi}\big(\widetilde{x}_{t+1}, x_t^\phi\big)-B_{\phi}\big(x_t, \widetilde{x}_t^\phi\big)\big]/ \eta_t.
\end{aligned}
\end{equation*}
The first inequality takes advantage of the optimality condition, and the second inequality is part of \cref{eq:1-order-x-bound}. 
Now we know that $\eta_t$ is non-increasing over time, $B_\phi$ is $L_{B_\phi}$-Lipschitz w.r.t. the first variable, and $L_{B_\phi} D_X$ is the supremum of $B_\phi$. Thus, 
\begin{equation*}
\begin{aligned}
\sum_{t=1}^{T}\Phi_t 
&\leq\frac{B_{\phi}\big(u_{0}, \widetilde{x}_{1}^\phi\big)}{\eta_0} +\sum_{t=1}^{T}\frac{1}{\eta_t}\left(B_{\phi}\big(u_t, \widetilde{x}_{t}^\phi\big)-B_{\phi}\big(u_{t-1}, \widetilde{x}_{t}^\phi\big)\right) +\sum_{t=1}^{T}\left(\frac{1}{\eta_{t}}-\frac{1}{\eta_{t-1}}\right)B_{\phi}\big(u_{t-1}, \widetilde{x}_{t}^\phi\big) \\
&\leq\frac{L_{B_\phi} D_X}{\eta_T}+\sum_{t=1}^{T}\frac{L_{B_\phi}}{\eta_t}\left\lVert u_t-u_{t-1}\right\rVert,
\end{aligned}
\end{equation*}
Note that \cref{eq:1-order-x-bound} can be relaxed as $f_t(x_t, y_t)-f_t(u_t, y_t)\leq\Phi_t+\nu_t^x$, summing over time yields
\begin{equation*}
\begin{aligned}
\dregret\,(u_{1:T})
\leq \frac{L_{B_\phi} D_X}{\eta_T}+\sum_{t=1}^{T}\frac{L_{B_\phi}}{\eta_t}\left\lVert u_t-u_{t-1}\right\rVert + \sum_{t=1}^{T} \nu_t^x.
\end{aligned}
\end{equation*}
Likewise, 
\begin{equation*}
\begin{aligned}
\dregret\,(v_{1:T})
\leq \frac{L_{B_\psi} D_Y}{\gamma_T}+\sum_{t=1}^{T}\frac{L_{B_\psi}}{\gamma_t}\left\lVert v_t-v_{t-1}\right\rVert + \sum_{t=1}^{T} \nu_t^y.
\end{aligned}
\end{equation*}
where $\nu_t^y=f_t(x_t,\widetilde{y}_{t+1})-h_t(x_t,\widetilde{y}_{t+1})+h_t(x_t,y_t)-f_t(x_t,y_t)-B_{\psi}\big(\widetilde{y}_{t+1}, y_t^\psi\big)/\gamma_t \geq 0$. \\
Let's go back to the focus on the $x$-player. 
The prescribed learning rate guarantees that
\begin{equation*}
\begin{aligned}
\dregret\,(u_{1:T})
\leq\epsilon+2\sum_{t=1}^{T}\nu_t^x.
\end{aligned}
\end{equation*}
On the one hand, $\nu_t^x\leq 2\rho(f_t,h_t)$ causes
\begin{equation}
\label{eqpf:OptIOMDA-x-bound-part1}
\begin{aligned}
\dregret\,(u_{1:T})\leq\epsilon+4\sum_{t=1}^{T}\rho(f_t,h_t).
\end{aligned}
\end{equation}
On the other hand, notice that 
\begin{equation*}
\begin{aligned}
\nu_t^x
&\leq \big\langle \nabla_x f_t(x_t, y_t)-\nabla_x h_t(\widetilde{x}_{t+1},y_t), x_t-\widetilde{x}_{t+1}\big\rangle-B_{\phi}\big(\widetilde{x}_{t+1}, x_t^\phi\big)/\eta_t \\
&\leq 2G_X\lVert x_t-\widetilde{x}_{t+1}\rVert-B_{\phi}\big(\widetilde{x}_{t+1}, x_t^\phi\big)/\eta_t 
\leq \min\big\{2 D_X G_X, 2\eta_t G_X^2\big\},
\end{aligned}
\end{equation*}
which implies that 
\begin{equation*}
\begin{aligned}
\left(\sum_{t=1}^{T}\nu_t^x\right)^2
&=\sum_{t=1}^{T}\left(\nu_t^x\right)^2+2\sum_{t=1}^{T}\nu_t^x\sum_{\tau=1}^{t-1}\nu_\tau^x 
=\sum_{t=1}^{T}\left(\nu_t^x\right)^2+2\sum_{t=1}^{T}\nu_t^x\bigg(\frac{L_{B_\phi}(D_X+\lambda)}{\eta_t}-\epsilon\bigg) \\
&\leq \sum_{t=1}^{T}4G_X^2 D_X^2+\sum_{t=1}^{T}4G_X^2L_{B_\phi}(D_X+\lambda). 
\end{aligned}
\end{equation*}
This results in the following regret bound:
\begin{equation}
\label{eqpf:OptIOMDA-x-bound-part2}
\begin{aligned}
\dregret\,(u_{1:T})\leq\epsilon+4G_X\sqrt{\left(D_X^2+L_{B_\phi} D_X+L_{B_\phi} \lambda\right)T}.
\end{aligned}
\end{equation}
Combining \cref{eqpf:OptIOMDA-x-bound-part1,eqpf:OptIOMDA-x-bound-part2} yields
\begin{equation*}
\begin{aligned}
\dregret\,(u_{1:T})\leq\epsilon+4\min\left\{\sum_{t=1}^{T}\rho(f_t,h_t), G_X\sqrt{\left(D_X^2+L_{B_\phi} D_X+L_{B_\phi} \lambda\right)T}\right\}.
\end{aligned}
\end{equation*}
Likewise, the individual regret of Player~2 satisfies
\begin{equation*}
\begin{aligned}
\dregret\,(v_{1:T})\leq\epsilon+4\min\left\{\sum_{t=1}^{T}\rho(f_t,h_t), G_Y\sqrt{\left(D_Y^2+L_{B_\psi} D_Y+L_{B_\psi} \mu\right)T}\right\}. 
\end{aligned}
\end{equation*}
Integrating the two individual regrets into D-DGap yields the desired result.
\end{proof}

\subsection{Proof of \cref{thm:expert-layer}}
\begin{proof}
The expert update can be rearranged as follows:
\begin{equation}
\label{eq:expert-layer}
\begin{aligned}
(\widehat{x}_{t},\widehat{y}_{t})&=\arg\min_{x\in X}\max_{y\in Y} \boldsymbol{w}_t^\mathrm{T}
\begin{bmatrix}
h_t(x,y), & h_t(x,\overline{y}_t) \\
h_t(\overline{x}_t,y), & h_t(\overline{x}_t,\overline{y}_t)
\end{bmatrix}
\boldsymbol{\omega}_t
+ \frac{w_t}{\eta_t}B_{\phi}\big(x, \widetilde{x}_t^{\phi}\big)
- \frac{\omega_t}{\gamma_t}B_{\psi}\big(y, \widetilde{y}_t^{\psi}\big), \\
\widetilde{x}_{t+1}&=\arg\min\nolimits_{x\in X}\eta_t \bigl[f_t(x,\widehat{y}_t),\ f_t(x,\overline{y}_t)\bigr]\boldsymbol{\omega}_t+B_{\phi}\big(x, \widetilde{x}_{t}^{\phi}\big), \\
\widetilde{y}_{t+1}&=\arg\max\nolimits_{y\in Y}\gamma_t \bigl[f_t(\widehat{x}_t,y),\ f_t(\overline{x}_t, y)\bigr]\boldsymbol{w}_t-B_{\psi}\big(y, \widetilde{y}_{t}^{\psi}\big).
\end{aligned}
\end{equation}
The first-order optimality condition of \cref{eq:expert-layer} implies that 
\begin{equation*}
\begin{aligned}
\bigl\langle \eta_t\bigl[\nabla_x h_t(\widehat{x}_{t},\widehat{y}_{t}),\ \nabla_x h_t(\widehat{x}_{t},\overline{y}_t)\bigl]\boldsymbol{\omega}_t + \widehat{x}_{t}^\phi - \widetilde{x}_t^\phi,\ \widehat{x}_{t} - x' \bigr\rangle \leq 0,\qquad&\forall x'\in X, \\
\bigl\langle \gamma_t\bigl[\nabla_y (-h_t)(\widehat{x}_{t},\widehat{y}_{t}),\ \nabla_y (-h_t)(\overline{x}_{t},\widehat{y}_t)\bigl]\boldsymbol{w}_t + \widehat{y}_{t}^\psi - \widetilde{y}_t^\psi,\ \widehat{y}_{t} - y' \bigr\rangle \leq 0,\qquad&\forall y'\in Y, \\
\bigl\langle \eta_t \bigl[\nabla_x f_t(\widetilde{x}_{t+1}, \widehat{y}_t),\ \nabla_x f_t(\widetilde{x}_{t+1},\overline{y}_t)\bigl]\boldsymbol{\omega}_t + \widetilde{x}_{t+1}^\phi - \widetilde{x}_t^\phi,\ \widetilde{x}_{t+1} - x' \bigr\rangle \leq 0, \qquad&\forall x'\in X, \\
\bigl\langle \gamma_t \bigl[\nabla_y (-f_t)(\widehat{x}_t,\widetilde{y}_{t+1}),\ \nabla_y f_t(\overline{x}_t,\widetilde{y}_{t+1})\bigl]\boldsymbol{w}_t + \widetilde{y}_{t+1}^\psi - \widetilde{y}_t^\psi,\ \widetilde{y}_{t+1} - y' \bigr\rangle \leq 0, \qquad&\forall y'\in Y.
\end{aligned}
\end{equation*}
The proof of this theorem can be established by suitably adapting the proof of Lemma~\ref{lem:OptOPPM}, incorporating the following substitutions while accounting for the predefined upper bound on the comparator sequence path length, which scales linearly with time $T$. The substitution rules are as follows:
\begin{equation*}
\setlength{\tabcolsep}{0.3em}
\begin{tabular}{rcl}
Variables in Proof of Lemma~\ref{lem:OptOPPM}& &Variables in This Proof \\[2pt]
$(x_t,y_t)$ & $\longrightarrow$ & $(\widehat{x}_{t},\widehat{y}_{t})$ \\[3pt]
$f_t(\,\cdot\,,y_t)$ and $f_t(x_t,\,\cdot\,)$ & $\longrightarrow$ & $\bigl[f_t(\,\cdot\,,\widehat{y}_t),\ f_t(\,\cdot\,,\overline{y}_t)\bigr]\boldsymbol{\omega}_t$ and $\bigl[f_t(\widehat{x}_t,\,\cdot\,),\ f_t(\overline{x}_t, \,\cdot\,)\bigr]\boldsymbol{w}_t$ \\[4pt]
$h_t(\,\cdot\,,y_t)$ and $h_t(x_t,\,\cdot\,)$ & $\longrightarrow$ & $\bigl[h_t(\,\cdot\,,\widehat{y}_t),\ h_t(\,\cdot\,,\overline{y}_t)\bigr]\boldsymbol{\omega}_t$ and $\bigl[h_t(\widehat{x}_t,\,\cdot\,),\ h_t(\overline{x}_t, \,\cdot\,)\bigr]\boldsymbol{w}_t$ 
\end{tabular}
\end{equation*}
For completeness, we provide a detailed proof below. 

Let's derive the upper bound for the right-hand side of the metric. Note that 
\begin{equation*}
\begin{aligned}
\boldsymbol{A}_t^{1,\colon\!} \boldsymbol{\omega}_t - f_t(u_t, y_t)
\leq\ &\bigl[f_t(\widehat{x}_t, \widehat{y}_t),\ f_t(\widehat{x}_t,\overline{y}_t)\bigr]\boldsymbol{\omega}_t-\bigl[h_t(\widehat{x}_t, \widehat{y}_t),\ h_t(\widehat{x}_t,\overline{y}_t)\bigr]\boldsymbol{\omega}_t \\
&+ \bigl[h_t(\widehat{x}_t, \widehat{y}_t),\ h_t(\widehat{x}_t,\overline{y}_t)\bigr]\boldsymbol{\omega}_t - \bigl[h_t(\widetilde{x}_{t+1}, \widehat{y}_t),\ h_t(\widetilde{x}_{t+1},\overline{y}_t)\bigr]\boldsymbol{\omega}_t \\
&+ \bigl[h_t(\widetilde{x}_{t+1}, \widehat{y}_t),\ h_t(\widetilde{x}_{t+1},\overline{y}_t)\bigr]\boldsymbol{\omega}_t - \bigl[f_t(\widetilde{x}_{t+1}, \widehat{y}_t),\ f_t(\widetilde{x}_{t+1},\overline{y}_t)\bigr]\boldsymbol{\omega}_t \\
&+ \bigl[f_t(\widetilde{x}_{t+1}, \widehat{y}_t),\ f_t(\widetilde{x}_{t+1},\overline{y}_t)\bigr]\boldsymbol{\omega}_t - \bigl[f_t(u_t, \widehat{y}_t),\ f_t(u_t,\overline{y}_t)\bigr]\boldsymbol{\omega}_t.
\end{aligned}
\end{equation*}
By using convexity and first-order optimality conditions, we obtain
\begin{subequations}
\begin{align}
\begin{aligned}
\bigl[h_t(\widehat{x}_t, \widehat{y}_t),\ h_t(\widehat{x}_t,\overline{y}_t)\bigr]\boldsymbol{\omega}_t
&- \bigl[h_t(\widetilde{x}_{t+1}, \widehat{y}_t),\ h_t(\widetilde{x}_{t+1},\overline{y}_t)\bigr]\boldsymbol{\omega}_t \\
&\leq \bigl\langle\bigl[\nabla_x h_t(\widehat{x}_{t},\widehat{y}_{t}),\ \nabla_x h_t(\widehat{x}_{t},\overline{y}_t)\bigl]\boldsymbol{\omega}_t,\ \widehat{x}_{t} - \widetilde{x}_{t+1} \bigr\rangle \\
&\leq \big\langle \widetilde{x}_t^\phi - \widehat{x}_{t}^\phi,\ \widehat{x}_{t} - \widetilde{x}_{t+1} \big\rangle / \eta_t \\
&=\bigl(B_{\phi}\big(\widetilde{x}_{t+1}, \widetilde{x}_t^\phi\big) - B_{\phi}\big(\widetilde{x}_{t+1}, \widehat{x}_t^\phi\big) - B_{\phi}\big(\widehat{x}, \widetilde{x}_t^\phi\big)\bigr) / \eta_t, \,
\end{aligned} \label{eq:1st-opt-h}\\
\begin{aligned}
\bigl[f_t(\widetilde{x}_{t+1}, \widehat{y}_t),\ f_t(\widetilde{x}_{t+1},\overline{y}_t)\bigr]\boldsymbol{\omega}_t &- \bigl[f_t(u_t, \widehat{y}_t),\ f_t(u_t,\overline{y}_t)\bigr]\boldsymbol{\omega}_t \\
& \leq \big\langle\bigl[\nabla_x f_t(\widetilde{x}_{t+1}, \widehat{y}_t),\ \nabla_x f_t(\widetilde{x}_{t+1},\overline{y}_t)\bigl]\boldsymbol{\omega}_t,\ \widetilde{x}_{t+1} - u_t\big\rangle \\
& \leq
\big\langle \widetilde{x}_{t}^\phi - \widetilde{x}_{t+1}^\phi,\ \widetilde{x}_{t+1} - u_t\big\rangle / \eta_t \\
&=\bigl(B_{\phi}\big(u_t, \widetilde{x}_{t}^\phi\big) - B_{\phi}\big(u_t, \widetilde{x}_{t+1}^\phi\big) - B_{\phi}\big(\widetilde{x}_{t+1}, \widetilde{x}_{t}^\phi\big)\bigr) / \eta_t. 
\end{aligned} \label{eq:1st-opt-f}
\end{align}
\end{subequations}
Now we have that 
\begin{equation}
\label{eq:x-metric}
\begin{aligned}
\boldsymbol{A}_t^{1,\colon\!} \boldsymbol{\omega}_t - f_t(u_t, y_t)
\leq\ &\bigl[f_t(\widehat{x}_t, \widehat{y}_t),\ f_t(\widehat{x}_t,\overline{y}_t)\bigr]\boldsymbol{\omega}_t-\bigl[h_t(\widehat{x}_t, \widehat{y}_t),\ h_t(\widehat{x}_t,\overline{y}_t)\bigr]\boldsymbol{\omega}_t \\
&+ \bigl[h_t(\widetilde{x}_{t+1}, \widehat{y}_t),\ h_t(\widetilde{x}_{t+1},\overline{y}_t)\bigr]\boldsymbol{\omega}_t - \bigl[f_t(\widetilde{x}_{t+1}, \widehat{y}_t),\ f_t(\widetilde{x}_{t+1},\overline{y}_t)\bigr]\boldsymbol{\omega}_t \\
&+ \bigl(B_{\phi}\big(u_t, \widetilde{x}_{t}^\phi\big) - B_{\phi}\big(u_t, \widetilde{x}_{t+1}^\phi\big) \bigr) / \eta_t \\
=\ &\bigl(B_{\phi}\big(u_t, \widetilde{x}_{t}^\phi\big) - B_{\phi}\big(u_t, \widetilde{x}_{t+1}^\phi\big)\bigr) / \eta_t + \delta_t^x,
\end{aligned}
\end{equation}
where $\delta_t^x\geq 0$. This can be obtained by adding \cref{eq:1st-opt-h} and the following inequality: 
\begin{equation*}
\begin{aligned}
\bigl[f_t(\widehat{x}_t, \widehat{y}_t),\ f_t(\widehat{x}_t,\overline{y}_t)\bigr]\boldsymbol{\omega}_t + B_{\phi}\big(\widehat{x}_t, \widetilde{x}_{t}^{\phi}\big)/ \eta_t
\geq
\bigl[f_t(\widetilde{x}_{t+1}, \widehat{y}_t),\ f_t(\widetilde{x}_{t+1},\overline{y}_t)\bigr]\boldsymbol{\omega}_t + B_{\phi}\big(\widetilde{x}_{t+1}, \widetilde{x}_{t}^{\phi}\big)/ \eta_t,
\end{aligned}
\end{equation*}
which corresponds to the optimality of $\widetilde{x}_{t+1}$. 
Summing \cref{eq:x-metric} over time yields
\begin{equation*}
\begin{aligned}
\sum_{t=1}^{T} \Bigl(\boldsymbol{A}_t^{1,\colon\!} \boldsymbol{\omega}_t - f_t(u_t, y_t)\Bigr)
\leq \sum_{t=1}^{T} \frac{1}{\eta_t}\Bigl(B_{\phi}\big(u_t, \widetilde{x}_{t}^\phi\big) - B_{\phi}\big(u_t, \widetilde{x}_{t+1}^\phi\big)\Bigr) + \sum_{t=1}^{T} \delta_t^x,
\end{aligned}
\end{equation*}
Due to the non-increasing nature of the learning rate $\eta_t$, $B_\phi$ is upper bounded by $L_{B_\phi} D_X$ and is $L_{B_\phi}$-Lipschitz with respect to its first variable. Therefore, we have that
\begin{equation*}
\begin{aligned}
&\sum_{t=1}^{T} \frac{1}{\eta_t}\Bigl(B_{\phi}\big(u_t, \widetilde{x}_{t}^\phi\big) - B_{\phi}\big(u_t, \widetilde{x}_{t+1}^\phi\big)\Bigr) \\
&\leq\sum_{t=1}^{T}\frac{1}{\eta_t}\left(B_{\phi}\big(u_t, \widetilde{x}_{t}^\phi\big)-B_{\phi}\big(u_{t-1}, \widetilde{x}_{t}^\phi\big)\right) 
+\frac{B_{\phi}\big(u_{0}, \widetilde{x}_{1}^\phi\big)}{\eta_1} +\sum_{t=2}^{T}\left(\frac{1}{\eta_{t}}-\frac{1}{\eta_{t-1}}\right)B_{\phi}\big(u_{t-1}, \widetilde{x}_{t}^\phi\big) \\
&\leq\frac{L_{B_\phi} D_X}{\eta_T}+\sum_{t=1}^{T}\frac{L_{B_\phi}}{\eta_t}\left\lVert u_t-u_{t-1}\right\rVert.
\end{aligned}
\end{equation*}
Applying the prescribed learning rate yields
\begin{equation*}
\begin{aligned}
\sum_{t=1}^{T} \Bigl(\boldsymbol{A}_t^{1,\colon\!} \boldsymbol{\omega}_t - f_t(u_t, y_t)\Bigr)
\leq \frac{L_{B_\phi}}{\eta_T}\left(D_X+P_T^u\right) + \sum_{t=1}^{T} \delta_t^x
\leq \epsilon + 2\sum_{t=1}^{T} \delta_t^x,
\end{aligned}
\end{equation*}
where $P_T^u = \sum_{t=1}^{T} \left\lVert u_t-u_{t-1}\right\rVert\leq D_X T$. 
Note that 
\begin{equation}
\label{eqpf:relax1}
\begin{aligned}
\delta_t^x=\ &\bigl[f_t(\widehat{x}_t, \widehat{y}_t),\ f_t(\widehat{x}_t,\overline{y}_t)\bigr]\boldsymbol{\omega}_t-\bigl[h_t(\widehat{x}_t, \widehat{y}_t),\ h_t(\widehat{x}_t,\overline{y}_t)\bigr]\boldsymbol{\omega}_t \\
&+ \bigl[h_t(\widetilde{x}_{t+1}, \widehat{y}_t),\ h_t(\widetilde{x}_{t+1},\overline{y}_t)\bigr]\boldsymbol{\omega}_t - \bigl[f_t(\widetilde{x}_{t+1}, \widehat{y}_t),\ f_t(\widetilde{x}_{t+1},\overline{y}_t)\bigr]\boldsymbol{\omega}_t \\
\leq\ &2\max_{x\in\{\widehat{x}_t,\overline{x}_t,\widetilde{x}_{t+1}\},y\in\{\widehat{y}_t,\overline{y}_t\}} \left\lvert f_t(x,y)-h_t(x,y)\right\rvert \\
\leq\ & 2\rho(f_t,h_t),
\end{aligned}
\end{equation}
So we have that 
\begin{equation*}
\begin{aligned}
\sum_{t=1}^{T} \Bigl(\boldsymbol{A}_t^{1,\colon\!} \boldsymbol{\omega}_t - f_t(u_t, y_t)\Bigr)
\leq \epsilon + 4\sum_{t=1}^{T} \rho(f_t,h_t).
\end{aligned}
\end{equation*}
Likewise, the upper bound for the left-hand side of the metric is as follows: 
\begin{equation*}
\begin{aligned}
\sum_{t=1}^{T} \Bigl(f_t(x_t, v_t) - \boldsymbol{w}_t^\mathrm{T} \boldsymbol{A}_t^{\colon\!,1}\Bigr)
\leq \epsilon + 4\sum_{t=1}^{T} \rho(f_t,h_t).
\end{aligned}
\end{equation*}
Adding the above two inequalities yields the desired conclusion.
\end{proof}

\subsection{Proof of \cref{thm:meta-layer}}
\begin{proof}
The meta update can be reformulated as follows:
\begin{equation}
\label{eq:meta-layer}
\begin{aligned}
(\boldsymbol{w}_{t},\boldsymbol{\omega}_{t})&=\arg\min\nolimits_{\boldsymbol{w}\in \bigtriangleup_2^\alpha}\max\nolimits_{\boldsymbol{\omega}\in \bigtriangleup_2^\alpha} \boldsymbol{w}^\mathrm{T}\boldsymbol{\mathit{\Lambda}}_{t}\boldsymbol{\omega} +\KL (\boldsymbol{w},\widetilde{\boldsymbol{w}}_t)/\theta_t -\KL (\boldsymbol{\omega},\widetilde{\boldsymbol{\omega}}_t)/\vartheta_t, \\[1.5pt]
\widetilde{\boldsymbol{w}}_{t+1}&=\arg\min\nolimits_{\boldsymbol{w}\in \bigtriangleup_2^\alpha} \langle \theta_t \boldsymbol{A}_t\boldsymbol{\omega}_t, \boldsymbol{w}\rangle+\KL (\boldsymbol{w},\widetilde{\boldsymbol{w}}_t), \\
\widetilde{\boldsymbol{\omega}}_{t+1}&=\arg\max\nolimits_{\boldsymbol{\omega}\in \bigtriangleup_2^\alpha} \langle \vartheta_t \boldsymbol{A}_t^\mathrm{T} \boldsymbol{w}_t, \boldsymbol{\omega}\rangle-\KL (\boldsymbol{\omega},\widetilde{\boldsymbol{\omega}}_t),
\end{aligned}
\end{equation}
where $\alpha=2/T$. 
\cref{eq:meta-layer} corresponds to a bilinear instance of \cref{eq:optoppm}. To leverage Lemma~\ref{lem:OptOPPM}, it is necessary to decompose the static duality gap. Let $\boldsymbol{1}=[1,1]^\mathrm{T}$. By inserting auxiliary representations $\boldsymbol{w} = \alpha\boldsymbol{1}/2 + (1 - \alpha)\boldsymbol{u}\in\bigtriangleup_2^\alpha$ and $\boldsymbol{\omega} = \alpha\boldsymbol{1}/2 + (1 - \alpha)\boldsymbol{v}\in\bigtriangleup_2^\alpha$, we obtain
\begin{equation}
\label{eqpf:meta-decomp}
\begin{aligned}
\sum_{t=1}^T \left(\boldsymbol{w}_t^\mathrm{T} \boldsymbol{A}_t \boldsymbol{v} -  \boldsymbol{u}^\mathrm{T} \boldsymbol{A}_t \boldsymbol{\omega}_t \right)
=\ & \sum_{t=1}^T \left(\boldsymbol{w}_t^\mathrm{T} \boldsymbol{A}_t \boldsymbol{\omega} -  \boldsymbol{w}^\mathrm{T} \boldsymbol{A}_t \boldsymbol{\omega}_t \right) \\
&+ \sum_{t=1}^T \boldsymbol{w}_t^\mathrm{T} \boldsymbol{A}_t (\boldsymbol{v}-\boldsymbol{\omega}) + \sum_{t=1}^T (\boldsymbol{w} - \boldsymbol{u})^\mathrm{T} \boldsymbol{A}_t \boldsymbol{\omega}_t,
\end{aligned}
\end{equation}
where
\begin{equation}
\label{eqpf:meta-const}
\begin{aligned}
\sum_{t=1}^T \boldsymbol{w}_t^\mathrm{T} \boldsymbol{A}_t (\boldsymbol{v}-\boldsymbol{\omega})
&\leq T  \left\lVert \boldsymbol{A}_t\right\rVert_{\infty} \left\lVert\alpha \boldsymbol{v} - \frac{\alpha}{2}\boldsymbol{1}\right\rVert_1 \leq 2\alpha T M = 4 M,\\
\sum_{t=1}^T (\boldsymbol{w} - \boldsymbol{u})^\mathrm{T} \boldsymbol{A}_t \boldsymbol{\omega}_t 
&\leq T \left\lVert\frac{\alpha}{2}\boldsymbol{1} - \alpha \boldsymbol{u}\right\rVert_1 \left\lVert \boldsymbol{A}_t\right\rVert_{\infty} \leq 2\alpha T M = 4 M,
\end{aligned}
\end{equation}
and according to Lemma~\ref{lem:OptOPPM}, 
\begin{equation}
\label{eqpf:meta-using-lemma}
\begin{aligned}
\sum_{t=1}^T \left(\boldsymbol{w}_t^\mathrm{T} \boldsymbol{A}_t \boldsymbol{\omega} -  \boldsymbol{w}^\mathrm{T} \boldsymbol{A}_t \boldsymbol{\omega}_t \right)
\leq O\left(\min\left\{ \sum_{t=1}^{T} \left\lVert\boldsymbol{A}_{t}-\boldsymbol{\mathit{\Lambda}}_{t}\right\rVert_\infty,\, \sqrt{(1 + \ln T) T}\right\}\right).
\end{aligned}
\end{equation}
In applying Lemma~\ref{lem:OptOPPM}, we consider only the static duality gap, implying that the path lengths of comparator sequences are constrained to zero. Additionally, in Lemma~\ref{lem:OptOPPM}, Fenchel couplings are bounded by constants, specifically $B_\phi \leq L_{B_\phi} D_X$ and $B_\psi \leq L_{B_\psi} D_Y$, allowing us to omit the constant terms $L_{B_\phi} D_X$ and $L_{B_\psi} D_Y$ in the D-DGap upper bound. However, in this proof, the KL divergence is bounded by $\ln T$, as demonstrated by the following:
\begin{equation*}
\begin{aligned}
0\leq\KL(\boldsymbol{a},\boldsymbol{b})=\boldsymbol{a}^\mathrm{T}\ln\frac{\boldsymbol{a}}{\boldsymbol{b}}\leq\ln\boldsymbol{a}^\mathrm{T}\frac{\boldsymbol{a}}{\boldsymbol{b}}\leq\ln\left\lVert\frac{\boldsymbol{a}}{\boldsymbol{b}}\right\rVert_\infty\leq\ln T, \qquad\forall\boldsymbol{a},\boldsymbol{b}\in\bigtriangleup_2^\alpha.
\end{aligned}
\end{equation*}
Consequently, the term $\ln T$ cannot be omitted from the upper bound of the static duality gap. 

To obtain the conclusion of this proof, we can further relax the prediction error term in \cref{eqpf:meta-using-lemma}:
\begin{equation}
\label{eqpf:relax2}
\begin{aligned}
\left\lVert\boldsymbol{A}_{t}-\boldsymbol{\mathit{\Lambda}}_{t}\right\rVert_\infty 
=\max_{x\in\{\widehat{x}_t,\overline{x}_t\},y\in\{\widehat{y}_t,\overline{y}_t\}} \left\lvert f_t(x,y)-h_t(x,y)\right\rvert
\leq \rho(f_t,h_t).
\end{aligned}
\end{equation}
Now combining \cref{eqpf:meta-decomp,eqpf:meta-const,eqpf:meta-using-lemma,eqpf:relax2} completes the proof.
\end{proof}

\subsection{Proof of Proposition~\ref{prop:G-Lipschitz}}
\begin{proof}
According to \cref{eq:monotone-G}, we show that
\begin{equation}
\label{eq:monotone-G-detail}
\begin{aligned}
\boldsymbol{G}(\boldsymbol{x})=
\begin{bmatrix}
\eta_t\,\omega\,\nabla_x h_t(x,y)+\eta_t (1-\omega)\,\nabla_x h_t(x,\overline y_t)+\nabla\phi(x)-\nabla\phi(\widetilde x_t)
\\[4pt]
\gamma_t\,w\,\nabla_y (-h_t)(x,y) + \gamma_t (1-w)\,\nabla_y (-h_t)(\overline x_t,y)+\nabla\psi(y)-\nabla\psi(\widetilde y_t)
\\[2pt]
\theta_t\,\omega\,(h_t(x,y)-h_t(\overline x_t,y))+\theta_t (1-\omega)\,(h_t(x,\overline y_t)-h_t(\overline x_t,\overline y_t))+\ln \frac{w(1-\widetilde{w}_t)}{\widetilde{w}_t(1-w)}
\\[2pt]
\vartheta_t\, w\,(h_t(x,\overline y_t)-h_t(x,y))+\vartheta_t (1-w)\,(h_t(\overline x_t,\overline y_t)-h_t(\overline x_t,y))+\ln\frac{\omega(1-\widetilde{\omega}_t)}{\widetilde{\omega}_t (1-\omega)}
\end{bmatrix}.
\end{aligned}
\end{equation}
To establish the Lipschitz continuity of $\boldsymbol{G}$, we split the difference into two parts:
\begin{equation*}
\begin{aligned}
\left\lVert\boldsymbol{G}(\boldsymbol{x}) - \boldsymbol{G}(\boldsymbol{x}')\right\rVert
\leq \left\lVert\boldsymbol{G}(x,y,w,\omega) - \boldsymbol{G}(x',y',w,\omega)\right\rVert
+\left\lVert\boldsymbol{G}(x',y',w,\omega) - \boldsymbol{G}(x',y',w',\omega')\right\rVert.
\end{aligned}
\end{equation*}
We first bound the $(x,y)$-difference. 
Note that $\boldsymbol{G}(x,y,w,\omega) - \boldsymbol{G}(x',y',w,\omega)$ has four block-coordinates involve differences of gradients and function values. 
For example, the first block is
\begin{equation*}
\begin{aligned}
&\left\lVert \eta_t\,\omega(\nabla_x h_t(x,y)-\nabla_x h_t(x',y'))+\eta_t (1-\omega)(\nabla_x h_t(x,\overline y_t)-\nabla_x h_t(x',\overline y_t))+\nabla\phi(x)-\nabla\phi(x') \right\rVert \\
&\leq
\eta_t\,\omega\, (L_{xx}\left\lVert x-x'\right\rVert+L_{xy}\left\lVert y-y'\right\rVert) + \eta_t (1-\omega)\,L_{xx}\left\lVert x-x'\right\rVert+L_\phi\left\lVert x-x'\right\rVert. 
\end{aligned}
\end{equation*}
Squaring and summing the four analogous estimates for the other blocks gives
\begin{equation*}
\begin{aligned}
2\left\lVert\boldsymbol{G}(x,y,w,\omega) - \boldsymbol{G}(x',y',w,\omega)\right\rVert^2
\leq C_x\left\lVert x-x'\right\rVert^2 + C_y\left\lVert y-y'\right\rVert^2,
\end{aligned}
\end{equation*}
where $C_x = 4 \bigl((\eta_t L_{xx}+L_\phi)^2 + \gamma_t^2 L_{yx}^2 + (\theta_t^2 + 4\,\vartheta_t^2)\,G_X^2\bigr)$, and $C_y = 4 \bigl((\gamma_t L_{yy}+L_\psi)^2 + \theta_t^2 L_{xy}^2 + (\vartheta_t^2 + 4\,\theta_t^2)\,G_Y^2\bigr)$. 
Next, we bound the $(w,\omega)$-difference. 
With $(x',y')$ fixed, consider $\boldsymbol{G}(x',y',w,\omega)-\boldsymbol{G}(x',y',w',\omega')$. 
Again bounding each of the four components via the Lipschitz continuity of $h_t$ and the derivative bound $\bigl|\tfrac{d}{du}\ln\tfrac{u}{1-u}\bigr|\leq T$, we obtain
\begin{equation*}
\begin{aligned}
2\left\lVert\boldsymbol{G}(x',y',w,\omega) - \boldsymbol{G}(x',y',w',\omega')\right\rVert^2 
\leq C_w\left\lvert w-w'\right\rvert^2 + C_\omega\left\lvert\omega-\omega'\right\rvert^2,
\end{aligned}
\end{equation*}
where $C_w = 2\,\gamma_t^2 L_{yx}^2 D_X^2 + 4\,\vartheta_t^2 C$, $C_\omega = 2\,\eta_t^2 L_{xy}^2 D_Y^2 + 4 \,\theta_t^2 C$, and $C = \min\big\{D_X^2 (L_{xx}D_X+L_{xy}D_Y)^2,\,D_Y^2 (L_{yx}D_X+L_{yy}D_Y)^2\big\}+T^2$. 
Combining both parts gives
\begin{equation*}
\begin{aligned}
&\left\lVert\boldsymbol{G}(\boldsymbol{x}) - \boldsymbol{G}(\boldsymbol{x}')\right\rVert^2 \\
&\leq \bigl(\left\lVert\boldsymbol{G}(x,y,w,\omega) - \boldsymbol{G}(x',y',w,\omega)\right\rVert
+\left\lVert\boldsymbol{G}(x',y',w,\omega) - \boldsymbol{G}(x',y',w',\omega')\right\rVert\bigr)^2 \\
&\leq 2\left\lVert\boldsymbol{G}(x,y,w,\omega) - \boldsymbol{G}(x',y',w,\omega)\right\rVert^2
+2\left\lVert\boldsymbol{G}(x',y',w,\omega) - \boldsymbol{G}(x',y',w',\omega')\right\rVert^2 \\
&\leq C_x\left\lVert x-x'\right\rVert^2
+C_y\left\lVert y-y'\right\rVert^2
+C_w\left\lvert w-w'\right\rvert^2
+C_\omega\left\lvert\omega-\omega'\right\rvert^2 \\
&\leq L^2\bigl(\left\lVert x-x'\right\rVert^2
+\left\lVert y-y'\right\rVert^2
+\left\lvert w-w'\right\rvert^2
+\left\lvert\omega-\omega'\right\rvert^2\bigr),
\end{aligned}
\end{equation*}
which implies that $\left\lVert\boldsymbol{G}(\boldsymbol{x}) - \boldsymbol{G}(\boldsymbol{x}')\right\rVert\leq L\left\lVert \boldsymbol{x}-\boldsymbol{x}'\right\rVert$. 
\end{proof}

\subsection{Proof of \cref{thm:overall}}
\begin{proof}
In the proofs of \cref{thm:expert-layer,thm:meta-layer}, the relaxed inequalities $\delta_t^x, \delta_t^y\leq 2\rho(f_t,h_t)$ and $\left\lVert\boldsymbol{A}_{t}-\boldsymbol{\mathit{\Lambda}}_{t}\right\rVert_\infty \leq \rho(f_t,h_t)$ are utilized (refer to \cref{eqpf:relax1,eqpf:relax2}). 
However, by appropriately setting the loss vector $\boldsymbol{L}_{t}$, these upper bounds can be tightened further, as follows: 
\begin{equation*}
\begin{aligned}
\delta_t^x, \  \delta_t^y, \  2\left\lVert\boldsymbol{A}_{t}-\boldsymbol{\mathit{\Lambda}}_{t}\right\rVert_\infty
\leq 2\left\langle \boldsymbol{L}_{t},\boldsymbol{\xi}_t\right\rangle.
\end{aligned}
\end{equation*}
Applying Lemma~\ref{lem:Campolongo-Hedge}, we derive: 
\begin{equation*}
\begin{aligned}
\sum_{t=1}^T\left\langle \boldsymbol{L}_{t},\boldsymbol{\xi}_t\right\rangle
&\leq \sum_{t=1}^T\left\langle \boldsymbol{L}_{t},\boldsymbol{1}_k\right\rangle+2\sqrt{2M\left(1+\ln T\right)\sum\nolimits_{t=1}^T\left\langle \boldsymbol{L}_{t},\boldsymbol{1}_k\right\rangle}+O\left(\ln T\right) \\
&=\sum_{t=1}^T\left\langle \boldsymbol{L}_{t},\boldsymbol{1}_k\right\rangle+O\left(\sqrt{\ln T}\right)\sqrt{\sum\nolimits_{t=1}^T\left\langle \boldsymbol{L}_{t},\boldsymbol{1}_k\right\rangle}+O\left(\ln T\right) \\
&\leq 2\sum_{t=1}^T\left\langle \boldsymbol{L}_{t},\boldsymbol{1}_k\right\rangle+O\left(\ln T\right) 
\leq2\sum_{t=1}^{T}\rho\left(f_t, h_t^k\right)+O\left(\ln T\right),\qquad \forall k=1,2,\cdots,d,
\end{aligned}
\end{equation*}
where $\boldsymbol{1}_k$ is a $d$-dimensional one-hot vector with the $k$-th element being 1. 
Given the arbitrariness of $k$, it follows that: 
\begin{equation*}
\begin{aligned}
\sum_{t=1}^T\left\langle \boldsymbol{L}_{t},\boldsymbol{\xi}_t\right\rangle
\leq 2\min_{k\in \{1,2,\cdots,d\}}\sum_{t=1}^{T}\rho\left(f_t, h_t^k\right)+O\left(\ln T\right).
\end{aligned}
\end{equation*}
Therefore, the term $\sum_{t=1}^T\rho(f_t,h_t)$ in the performance bounds of both the meta layer and expert layer can be replaced with \[\widetilde{O}\left(\min_{k \in \{1,2,\cdots,d\}}\sum_{t=1}^T\rho(f_t,h_t^k)\right),\] resulting in the following D-DGap upper bound: 
\begin{equation*}
\begin{aligned}
\ddgap\,(u_{1:T},v_{1:T}) \leq \widetilde{O}\left(\min\left\{\min_{k \in \{1,2,\cdots,d\}}\sum_{t=1}^T \rho(f_t, h_t^k),\ \sqrt{\left(1 + \min\{P_T, C_T\}\right)T}\right\}\right),
\end{aligned}
\end{equation*}
which completes the proof.
\end{proof}

The following lemma can be referred to as the static version of Corollary~B.0.1 in \citet{Campolongo2021closer}.
\begin{lemma}[Static Regret for Clipped Hedge]
\label{lem:Campolongo-Hedge}
Let $\bigtriangleup_d^\alpha$ be a $d$-dimensional $\alpha$-clipped simplex, $T\geq d$ and $\alpha=d/T$. 
Assume that all bounded linear losses satisfy $\boldsymbol{L}_{t}\geq 0$ and $\max_{t\in 1:T}\lVert \boldsymbol{L}_t\rVert_\infty=L_\infty$. 
If $\boldsymbol{\xi}_t$ follows the clipped Hedge: 
\begin{equation*}
\begin{aligned}
\boldsymbol{\xi}_{t+1} = \arg\min_{\boldsymbol{\xi} \in \bigtriangleup_d^a} \zeta_t \left\langle \boldsymbol{L}_t, \boldsymbol{\xi} \right\rangle + \KL(\boldsymbol{\xi}, \boldsymbol{\xi}_t),
\end{aligned}
\end{equation*}
where the learning rate $\zeta_t$ is determined by the following equations: 
\begin{equation*}
\begin{aligned}
\textstyle \zeta_t = (\ln T)\big/\bigl(\epsilon + \sum_{\tau=1}^{t-1}\Delta_{\tau}\bigr), 
\qquad\epsilon >0,\qquad
\Delta_{t} = \left\langle \boldsymbol{L}_{t}, \boldsymbol{\xi}_t-\boldsymbol{\xi}_{t+1}\right\rangle -\KL (\boldsymbol{\xi}_{t+1},\boldsymbol{\xi}_{t})/\zeta_t>0. 
\end{aligned}
\end{equation*}
Then we have that
\begin{equation*}
\begin{aligned}
\sum_{t=1}^T\bigl\langle \boldsymbol{L}_{t}, \boldsymbol{\xi}_{t}-\boldsymbol{u}\bigr\rangle
\leq 2\sqrt{(1+\ln T)L_{\infty}\sum\nolimits_{t=1}^T\bigl\langle \boldsymbol{L}_{t}, \boldsymbol{u}\bigr\rangle}+O\left(\ln T\right),\qquad\forall\boldsymbol{u}\in\bigtriangleup_d. 
\end{aligned}
\end{equation*}
\end{lemma}



\vskip 0.2in
\bibliography{reference}

\end{document}